\documentclass[10pt,journal,compsoc]{IEEEtran}

\ifCLASSOPTIONcompsoc 
  \usepackage[nocompress]{cite}
\else
  \usepackage{cite}
\fi

\ifCLASSINFOpdf
\else
\fi

\usepackage{graphicx,wrapfig}
\usepackage{subfigure}
\usepackage{enumitem}
\usepackage{pifont}
\usepackage{multirow}
\usepackage{tabularx}
\usepackage{amsmath,amssymb,amsfonts,mathrsfs,amsthm}
\usepackage{makecell}
\usepackage{booktabs}
\usepackage{url}
\usepackage[ruled]{algorithm2e}
\usepackage{extarrows}
\usepackage{bbm}
\newcommand{\modelname}{EvoGOOD}
\newcommand{\ie}{{i.e.}}

\theoremstyle{definition}
\newtheorem{asm}{Assumption}

\newtheorem{prop}{Proposition}

\newtheorem{Lemma}{Lemma}

\usepackage{diagbox}

\newcommand{\change}[1]{\textcolor{black}{#1}}
\newcommand{\changecolor}{black}

\usepackage{pifont}
\usepackage{bm}
\usepackage{pifont}
\usepackage{hyperref}

\usepackage{xcolor}
\usepackage{caption}

\usepackage[normalem]{ulem}
\useunder{\uline}{\ul}{}
\usepackage{array}
\usepackage[hang]{footmisc}
\usepackage{breakurl}
\usepackage{adjustbox}
\usepackage{array}
\usepackage{color,colortbl}

\makeatletter
\def\UrlAlphabet{%
      \do\a\do\b\do\c\do\d\do\e\do\f\do\g\do\h\do\i\do\j%
      \do\k\do\l\do\m\do\n\do\o\do\p\do\q\do\r\do\s\do\t%
      \do\u\do\v\do\w\do\x\do\y\do\z\do\A\do\B\do\C\do\D%
      \do\E\do\F\do\G\do\H\do\I\do\J\do\K\do\L\do\M\do\N%
      \do\O\do\P\do\Q\do\R\do\S\do\T\do\U\do\V\do\W\do\X%
      \do\Y\do\Z}
\def\UrlDigits{\do\1\do\2\do\3\do\4\do\5\do\6\do\7\do\8\do\9\do\0}
\g@addto@macro{\UrlBreaks}{\UrlOrds}
\g@addto@macro{\UrlBreaks}{\UrlAlphabet}
\g@addto@macro{\UrlBreaks}{\UrlDigits}
\makeatother
\begin{document}
\title{Evolving Graph Learning for Out-of-Distribution Generalization in Non-stationary Environments}

\author{Qingyun~Sun,~\IEEEmembership{Member,~IEEE},~Jiayi~Luo,~Haonan~Yuan,~Xingcheng~Fu,~Hao~Peng,~\IEEEmembership{Member,~IEEE},\\Jianxin~Li,~\IEEEmembership{Senior Member,~IEEE},~and~Philip~S.~Yu,~\IEEEmembership{Life~Fellow,~IEEE}%
\IEEEcompsocitemizethanks{\IEEEcompsocthanksitem Q. Sun, J. Luo, H. Yuan, H. Peng, and J. Li are with Beijing Advanced Innovation Center for Big Data and Brain Computing, School of Computer Science and Engineering, Beihang University, Beijing 100083, China.
\protect\\
E-mail: \{sunqy, luojy, yuanhn, penghao, lijx\}@buaa.edu.cn;
\IEEEcompsocthanksitem X. Fu is with the Key Lab of Education Blockchain and Intelligent Technology, Ministry of Education, Guangxi Normal University, China. E-mail: fuxc@buaa.edu.cn.%
\IEEEcompsocthanksitem PS. Yu is with the Department of Computer Science, University of Illinois at Chicago, Chicago 60607, USA. E-mail: psyu@uic.edu.}%
\thanks{Manuscript received Mar. 1, 2024.  }}

\markboth{Journal of \LaTeX\ Class Files,~Vol.~14, No.~8, August~2021}%
{Sun \MakeLowercase{\textit{et al.}}: Bare Demo of IEEEtran.cls for Computer Society Journals}

\IEEEtitleabstractindextext{%
\begin{abstract}
Graph neural networks have shown remarkable success in exploiting the spatial and temporal patterns on dynamic graphs. 
However, existing GNNs exhibit poor generalization ability under distribution shifts, which is inevitable in dynamic scenarios. 
As dynamic graph generation progresses amid evolving latent non-stationary environments, it is imperative to explore their effects on out-of-distribution (OOD) generalization.
This paper proposes a novel \textbf{Evo}lving \textbf{G}raph Learning framework for \textbf{OOD} generalization (\textbf{EvoGOOD}) by environment-aware invariant pattern recognition. 
Specifically, we first design an environment sequential variational auto-encoder to model environment evolution and infer underlying environment distribution. 
Then, we introduce a mechanism for environment-aware invariant pattern recognition, tailored to address environmental diversification through inferred distributions.
Finally, we conduct fine-grained causal interventions on individual nodes using a mixture of instantiated environment samples. This approach helps to distinguish spatio-temporal invariant patterns for OOD prediction, especially in non-stationary environments.
Experimental results demonstrate the superiority of \modelname~on both real-world and synthetic dynamic datasets under distribution shifts. 
To the best of our knowledge, it is the first attempt to study the dynamic graph OOD generalization problem from the environment evolution perspective. 
\end{abstract}

}

\maketitle

\IEEEdisplaynontitleabstractindextext

\IEEEpeerreviewmaketitle

\IEEEraisesectionheading{\section{Introduction}\label{sec:introduction}}

\IEEEPARstart{G}{\change{raph}} \change{representation learning has been extensively studied, ranging from static knowledge graph embedding approaches~\cite{wang2022duality} to dynamic graph neural networks (DGNNs), which}
have achieved great success in many applications across different domains including social network~\cite{cordeiro2016dynamic}, recommender system~\cite{wu2022graph}, fraud detection~\cite{ma2021comprehensive}, etc. 
DGNNs combine the merits of GNN-based models and sequential-based models and show excellent expressive power in exploring the highly complex spatio-temporal patterns in dynamic graphs.

Although a vast number of emerging methods and techniques have flourished in recent years, 
many of them are based on the assumption of in-distribution, 
meaning that the testing data is assumed to be identically distributed to the training data. 
However, the GNN models may encounter significantly different distributions of features, structure, and labels, which is inevitable, especially in dynamic scenarios. 
Recent works~\cite{shen2021towards,zhu2021shift,hendrycks2021many} have demonstrated that existing works exhibit poor generalization ability when faced with unknown distribution shifts. 
Distribution shifts can cause GNNs to overfit to labeled samples and make their predictions on unseen samples prone to error, resulting in unstable predictions. 
Therefore, developing GNNs that are robust against distribution shifts is crucial for their applicability in real-world scenarios.

Out-of-distribution (OOD) generalization in graph data has received extensive research attention in the fields of graph data in recent years~\cite{shen2021towards,li2022ood,chen2022learning,gui2022good,fan2021generalizing}. 

From the causal-based theories~\cite{pearl2009causal, pearl2010causal,pearl2018book}, the generation of graphs can be described using a Structural Causal Model (SCM)~\cite{pearl2009causal} as shown in Fig. ~\ref{fig:SCM}, with arrows indicating causal relationships between variables. 
In the field of OOD generalization, there is a widespread consensus~\cite{gagnon2022woods, arjovsky2019invariant, rosenfeld2020risks, wu2022discovering, chang2020invariant, ahuja2020invariant, mitrovic2020representation} that the correlations between labels and causal components ($C$) of the latent features are invariant across different data distributions in both training and testing phases, while other parts of the features constitute the variant/spurious component ($S$). 
It is widely recognized that spurious correlations between $C\leftrightarrow S$ within and across graph snapshots are detrimental to generalization performance under distribution shifts.
These correlations need to be filtered out through careful investigation of the impact of latent environment $\mathbf{e}$~\cite{pearl2009causal, peters2017elements}. 
The majority of existing works are proposed for static graphs and focus on finding invariant patterns in terms of features and structures as shown in Fig.~\ref{fig:SCM}, ignoring the spurious correlations in spatio-temporal interactions. 
In addition, they assume that the graph generation environment follows a fixed distribution~\cite{li2022graphde,zhang2022dynamic}. 
However, real-world environments are usually non-stationary~\cite{sugiyama2012machine,park2020hop}, \ie, the environment distribution changes over time. 
Since the generation of graphs is affected by many unknown factors in the dynamic environment, the spatio-temporal distribution shifts in dynamic graphs are more complex to handle. 
\begin{figure}
    \centering
    \includegraphics[width=1\linewidth]{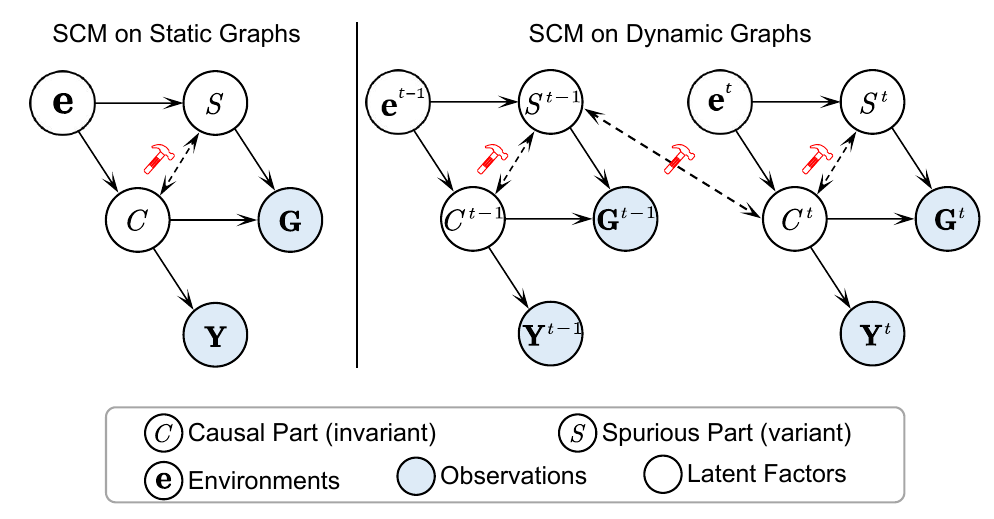}
    \caption{SCM models for graph OOD generalization.}
    \label{fig:SCM}
\end{figure}

This paper studies the problem of identifying invariant spatio-temporal patterns in dynamic graphs under non-stationary environments, a topic that has not been explored in the existing literature. 
However, it is non-trivial with the following two challenges:
1) How to appropriately model and infer the complex evolutionary environments on dynamic graphs? 
2) How to recognize the invariant patterns under the unknown spatio-temporal distribution shifts of the non-stationary environments? 
To alleviate the aforementioned issues, we can take advantage of the graph evolution patterns along with the environment evolution. 
For example, graduation and job-hopping would affect the interactions in collaboration networks and friendship networks, which can be exploited as the underlying environments to predict future coauthors or to recommend possible friends. 

\textbf{Present work.} 
This paper proposes a novel \textbf{Evo}lving \textbf{G}raph Learning framework for \textbf{O}ut-\textbf{O}f-\textbf{D}istribution Generalization named \textbf{\modelname}~to learn invariant prediction across spatio-temporal distribution shifts. 
First, to shed light on the evolving environment, we design an environment sequential auto-encoder (ESVAE) to model and infer environments by variational inference. 
This approach enables us to predict the future environment and identify the corresponding invariant patterns. 
Next, we introduce a novel environment-aware invariant pattern recognition mechanism that satisfies the \textit{Invariance Property} and \textit{Sufficient Condition}, and offers theoretical assurances for out-of-distribution prediction.
Lastly, we use a fine-grained causal intervention mechanism for each node with inferred environment distributions for better generalization. 
In this way, \modelname~can generalize well in the non-stationary environment by recognizing the environment-related invariant patterns. 

The initial iteration of this research was presented in the proceedings of the Conference on Neural Information Processing System 2023~\cite{yuan2023environment}. %
This journal version incorporates several enhancements that elevate the framework, focusing on significant aspects. 
First, \modelname~learns both static and dynamic factors in non-stationary environments with a novel environment sequential variational autoencoder instead of using a single distribution to model the environments as in the previous version, which can capture the evolving patterns of dynamic environments and generalize better to the unseen future environments. 
In this way, \modelname~can not only model the intrinsic changing dynamics of the observed environments but also predict the future environment distribution of finite steps and generate convincing instances for causal inference. 
Second, \modelname~performs fine-grained node-wise causal interventions with not only the historical samples but also the generated samples from the current environment. 
With the generated samples that can better represent the potentially shifted future environments, EvoGOOD can achieve better generalization capability, and extend the generalization boundaries to unseen testing distributions. 
Third, \modelname~uses a simple temporal convolution for dynamic graphs instead of the disentangled graph encoding method in the previous version, showing better computational and space complexity. 
Last, we also conduct more experiments on not only the link prediction task as in the preliminary version but also on the node classification task to further explore the generalization ability of \modelname. 
The contributions can be summarized as follows:
\begin{itemize}[leftmargin=1.5em]
    \item 
    We introduce a novel OOD generalization framework \modelname~for dynamic graphs, which leverages spatio-temporal invariant patterns in relation to environments.
    To our best knowledge, 
    this is the first attempt to investigate the influence of environment evolution on dynamic graphs during distribution shifts. 
    \item 
    We design a novel environment sequential auto-encoder to enhance \modelname's extrapolation capability for distribution shifts in non-stationary environments.
    This enables \modelname~to recognize the environment-aware invariant patterns and perform the node-wise fine-grained causal intervention to enhance the generalization ability. 
    \item 
    Extensive experiments validate the superior generalization ability of \modelname~over state-of-the-art baselines in the future link prediction and node classification tasks, utilizing both real-world and synthetic datasets. 

\end{itemize}

The rest of this paper is structured as follows.
Section~\ref{sec:related} discusses the related works about dynamic graph learning, graph OOD generalization methods, and environment learning for OOD generalization. 
Section~\ref{sec:formulation} gives the problem formulations. 
Section~\ref{sec:method} reviews the overall framework of~\modelname~following the environment ``Modeling-Inferring-Discriminating-Generalizing'' paradigm, respectively. 
Section~\ref{sec:exp} describes the experimental methodology and the results. 
Section~\ref{sec:conclusion} concludes this work. 

\section{Related Work}
\label{sec:related}
This section will discuss the related works including dynamic graph learning, the disentangled graph OOD generalization, invariant graph OOD generalization, and graph OOD generalization for dynamic scenarios. 

\subsection{Dynamic Graph Learning}
How to handle the spatial-temporal patterns of nodes and edges is the main issue in dynamic graph learning. 
Existing works~\cite{seo2018structured,sankar2020dysat} typically are a combination of GNNs and sequential models to regulate the node embeddings and learn the spatial-temporal dynamics. 
For example, 
GCRN~\cite{seo2018structured} incorporates the Graph Convolutional Network and GRU~\cite{cho2014learning} to learn the spatial-temporal relations. 
DySAT~\cite{sankar2020dysat} captures the dynamic structural patterns through a unified self-attention mechanism.

Encoding the evolution patterns of dynamic graphs has attracted recent research interest. 
DynamicTriad~\cite{zhou2018dynamic} utilizes the triad closure process to simulate the graph evolution. 
HTNE~\cite{zuo2018embedding} captures the sequence of neighborhood formation during the evolutionary process using a Hawkes process with time-dependent weights.
HoMo-DyHNE~\cite{ji2023higher} incorporates node features and the Hawkes Process into a skip-gram model to learn the evolution of complex patterns for heterogeneous networks. 
EPNE~\cite{wang2020epne} proposes the causal convolutions to learn the periodic linkage evolution patterns. 
EvolveGCN~\cite{pareja2020evolvegcn} utilizes an RNN to dynamically adjust the GCN parameters for model adaptation at each timestamp. 
This paper aims to investigate the invariant patterns that evolve during the graph evolution.

\begin{table*}[t]
  \caption{Comparison of graph out-of-distribution generalization methods.}
  \resizebox{\linewidth}{!}{
  \centering
  \renewcommand\arraystretch{1}
  \begin{tabular}{clccccccc}
    \toprule
    \multicolumn{2}{c}{\multirow{2}[4]{*}{\textbf{Model}}} & \multicolumn{1}{c}{\multirow{2}[4]{*}{\makecell{\textbf{factor}\\\textbf{disentangling}}}} & \multicolumn{1}{c}{\multirow{2}[4]{*}{\makecell{\textbf{invariant pattern}\\\textbf{recognition}}}} & \multicolumn{2}{c}{\textbf{causal inference}} & \multicolumn{3}{c}{\textbf{environment extrapolation}} \\
\cmidrule(r){5-6} \cmidrule(r){7-9}   \multicolumn{2}{c}{} &       &       & \multicolumn{1}{c}{coarse-grained} & \multicolumn{1}{c}{fine-grained} & \multicolumn{1}{c}{\makecell{environment\\modeling}} & \multicolumn{1}{c}{\makecell{environment\\diversification}} & \multicolumn{1}{c}{\makecell{environment\\evolution}} \\
    \midrule
    \multirow{7}[38]{*}{\rotatebox{90}{\makecell{\textbf{static}}}} & \multicolumn{1}{l}{\makecell[l]{DisenGCN~\cite{ma2019disentangled}, IPGDN~\cite{liu2020independence}, DisC~\cite{fan2022debiasing},\\FactorGCN~\cite{yang2020factorizable}, NED-VAE~\cite{guo2020interpretable}, DGCL~\cite{li2021disentangled},\\IDGCL~\cite{li2022disentangled}}} & \checkmark &       &       &       &       &       &  \\
\cmidrule{2-9}          & \multicolumn{1}{l}{DGNN~\cite{fan2022debiased}, CAL~\cite{sui2022causal}, CFLP~\cite{zhao2022learning}, Gem~\cite{lin2021generative}} &       &       & \checkmark &       &       &       &  \\
\cmidrule{2-9}          & \multicolumn{1}{l}{\makecell[l]{GRAND~\cite{feng2020graph}, FLAG~\cite{kong2022robust}, GraphCL~\cite{you2020graph},\\GREA~\cite{liu2022graph}, DPS~\cite{yu2022finding}, Mixup~\cite{zhang2017mixup},\\GraphMix~\cite{verma2021graphmix}, G-Mixup~\cite{wang2021mixup}, NodeAug~\cite{wang2020nodeaug}}} &       &       &       &       &       & \checkmark &  \\
\cmidrule{2-9}          & \multicolumn{1}{l}{OOD-GNN~\cite{li2022ood}, StableGNN~\cite{fan2021generalizing}} &       & \checkmark & \checkmark &       &       &       &  \\
\cmidrule{2-9}          & \multicolumn{1}{l}{\makecell[l]{GAug~\cite{zhao2021data}, MH-Aug~\cite{park2021metropolis}, KDGA~\cite{wu2022knowledge},\\AdvCA~\cite{sui2022adversarial}, LA-GNN~\cite{liu2022local}, ifMixup~\cite{guo2022intrusionfree},\\$\mathcal{G}$-Mixup~\cite{han2022g}}} &       &       &       &       & \checkmark & \checkmark &  \\
\cmidrule{2-9}          & \multicolumn{1}{l}{\makecell[l]{SR-GNN~\cite{zhu2021shift}, GSAT~\cite{miao2022interpretable}, EERM~\cite{wu2022handling}, GIL~\cite{li2022learning},\\SizeShiftReg~\cite{buffelli2022sizeshiftreg}, RGCL~\cite{li2022let}, INL~\cite{li2023invariant},\\FLOOD~\cite{liu2023flood}}} &       & \checkmark &       &       & \checkmark & \checkmark &  \\
\cmidrule{2-9}          & \multicolumn{1}{l}{\makecell[l]{DSE~\cite{wu2022deconfounding}, DIR~\cite{wu2022discovering}, CIGA~\cite{chen2022learning}, gMPNN~\cite{zhou2022ood},\\E-invariant GR~\cite{bevilacqua2021size}, LECI\cite{gui2024joint}, CAL+~\cite{sui2024enhancing},\\CaNet~\cite{wu2024graph}}} &       & \checkmark & \checkmark &       & \checkmark & \checkmark &  \\
    \midrule
    \multirow{3}[20]{*}{\rotatebox{90}{\makecell{\textbf{dynamic}}}} & DIDA~\cite{zhang2022dynamic}, SILD~\cite{zhang2023spectral} &       & \checkmark & \checkmark &       &       &  &  \\
\cmidrule{2-9}          & EAGLE~\cite{yuan2023environment} &   \checkmark    & \checkmark &       & \checkmark & \checkmark & \checkmark &  \\
\cmidrule{2-9}          & \change{EpoD~\cite{yang2024improving}} &       & \change{\checkmark} &   \change{\checkmark}    &  & \change{\checkmark} &  &  \\
\cmidrule{2-9}          & \change{PURE~\cite{wu2024pure}} &       &  &     &  & \change{\checkmark} &  &  \\
\cmidrule{2-9}          & \change{ProST~\cite{xia2025prost}} &       & \change{\checkmark} &     &  & \change{\checkmark} &  &  \\
\cmidrule{2-9}          & \textbf{\modelname~(ours)} & \checkmark & \checkmark &       & \checkmark & \checkmark & \checkmark & \checkmark \\
    \bottomrule
    \end{tabular}%
  \label{tab:comparison}%
}
\end{table*}%

\subsection{Disentangled Graph OOD Generalization}
The generation of real-world graphs typically involves complex processes influenced by numerous underlying factors~\cite{rojas2018invariant, arjovsky2019invariant}. 
Disentangled graph representation learning is a technique that seeks to extract representations that isolate the various distinctive and informative factors that underpin graph data. 
These factors are then captured in different components of the factorized vector representations. 
Existing works can be categorized into supervised factor disentangling~\cite{ma2019disentangled, liu2020independence, fan2022debiasing, yang2020factorizable}, unsupervised factor disentangling~\cite{guo2020interpretable}, and self-supervised contrastive factor disentangling~\cite{li2021disentangled, li2022disentangled}. 
Studies have shown that such disentangled representations are of high quality and can contribute to improving OOD generalization capabilities.

\subsection{Invariant Graph OOD Generalization}
Deep learning models frequently capture predictive correlations within the observed samples. However, the patterns learned may not always align with unseen data. Invariant learning seeks to identify these patterns that exhibit minimal variation between training and testing, thereby enabling the generation of informative and discriminative representations that underpin stable predictions~\cite{creager2021environment, li2021learning, zhao2019learning}. 

Graph machine learning models often capitalize on statistical correlations within the training set, which may be spurious and lead to improved training accuracy. However, these models' performance can suffer significantly when deployed in OOD testing scenarios, as the spurious correlations they depend on may not hold in new, unobserved environments.
Encouraged by causal theories, invariant learning harnesses the power of Structural Causal Models (SCMs)~\cite{pearl2009causal} to eliminate spurious correlations through interventions or counterfactuals, enforcing significant causal dependencies, and approaches the OOD generalization challenge from a more theoretical standpoint, thus exposing its considerable potential~\cite{gagnon2022woods, arjovsky2019invariant, rosenfeld2020risks, chang2020invariant, ahuja2020invariant, mitrovic2020representation}. For graph data, \cite{li2022ood, fan2021generalizing, wu2022deconfounding, wu2022discovering, chen2022learning, zhou2022ood, bevilacqua2021size, gui2024joint, sui2024enhancing, wu2024graph} perform causal inference to remove spurious correlations such that contribute to determining the predictive dependency on the causal invariant patterns. Additionally, \cite{fan2022debiased, sui2022causal, zhao2022learning, lin2021generative} apply confounder balancing techniques to carry out explicit causal inference with better explainability directly. However, these methods only carry out coarse causal inference for node level or graph level and lack individual and adaptive invariant pattern mining. 

In contrast, \cite{zhu2021shift, miao2022interpretable, wu2022handling, li2022learning, buffelli2022sizeshiftreg, li2022let, li2023invariant, liu2023flood} do not rely on the causal inference to learn invariant patterns. 
They depend on generative or heuristically deterministic modules to optimize a learnable invariant pattern, leveraging the principle of invariance to address the graph OOD generalization challenge.

\subsection{Environment Learning for Graph OOD Generalization}
Environment plays an important role in OOD generalization, especially in the OOD generalization of graphs. 
Given that the environment labels are not available in real-world scenarios, the environments present in graph data are typically considered as latent factors that contribute to data distribution shifts.
The environment extrapolation method is usually combined with the aforementioned disentangled representation learning method and invariant representation learning method to improve the graph OOD generalization performance. 

According to the depth of utilizing the latent environment, existing works fall into two categories: random-data-augmentation-based environment diversification and distribution-modeling-based environment diversification. For example, \cite{feng2020graph, kong2022robust, you2020graph, liu2022graph, yu2022finding, zhang2017mixup, verma2021graphmix, wang2021mixup, wang2020nodeaug} are random-data-augmentation-based environment diversification methods that diversify the latent environments by randomly or heuristically mixing and generating new data samples. These methods are well known for their simple but effective performance in most cases. For distribution-modeling-based environment diversification methods, ~\cite{zhao2021data, park2021metropolis, wu2022knowledge, sui2022adversarial, liu2022local, guo2022intrusionfree, han2022g} first explicitly model the distribution of latent environments and generate necessary data instantiations for environment diversifying by sampling from the inferred environment distribution. In addition, the diversified environments with their generated samples are utilized to gain better invariant pattern recognition~\cite{zhu2021shift, miao2022interpretable, wu2022handling, li2022learning, buffelli2022sizeshiftreg, li2022let, li2023invariant, liu2023flood, wu2022deconfounding, wu2022discovering, chen2022learning, zhou2022ood, bevilacqua2021size, gui2024joint, sui2024enhancing, wu2024graph}.

\subsection{Graph OOD Generalization for Dynamic Scenario}
Existing static graph OOD methods are not tailored to dynamic graph data. 
DIDA~\cite{zhang2022dynamic} is the first work tackling the distribution shifts by a spatio-temporal attention network to determine the time-invariant patterns with coarse-grained causal inference techniques. 
Similar to DIDA~\cite{zhang2022dynamic}, SILD~\cite{zhang2023spectral} addresses the dynamic graph OOD problem by identifying and leveraging invariant patterns in the spectral domain. 
However, these two methods omit the importance of exploring the impact of environments on dynamic graphs, which are of vital importance in the dynamic scenario, as dynamic environments are coupled in both spatial and temporal perspectives. 
EAGLE~\cite{yuan2023environment} first models the dynamic and complex intertwined environments and then utilizes spatio-temporal invariant patterns within these environments. 
However, it overlooks the evolution of these environments, which is essential for identifying invariant patterns in dynamic graph settings. 
Furthermore, existing studies assume that the OOD component conforms to a fixed distribution~\cite{li2022graphde}, which is not applicable in dynamic settings.

\change{Recent studies complement dynamic OOD generalization with prompt mechanisms. EpoD~\cite{yang2024improving} proposes a self-prompted scheme to infer latent environments and utilize them by treating dynamic subgraphs as mediators in a structural causal pathway, thereby enhancing robustness to temporal environment shifts. PURE~\cite{wu2024pure} addresses OOD adaptation in dynamical systems by learning time-evolving prompts governed by a graph ODE: prompts are initialized from multi-view frequency-domain contexts and then evolved via ODE-based interpolation to adapt forecasting models under temporal distribution shifts, with mutual-information regularization to improve robustness. ProST~\cite{xia2025prost} focuses on spatio-temporal prediction and leverages dynamic graph pre-training to build a premise graph, upon which subgraph prompts (optimized via meta-learning) are used to infer explicit future snapshots.}

\change{While these works demonstrate the promise of prompt-based adaptation or snapshot inference in dynamic scenarios, they do not explicitly model environment evolution as a distributional process nor provide intervention-based, node-wise invariant-pattern identification.}

Table~\ref{tab:comparison} compares our \modelname~with existing graph OOD generalization methods in terms of the graph OOD generalization capabilities, \ie, factor disentangling, invariant pattern recognition, causal inference, and environment extrapolation. 
Nevertheless, the application of invariant learning to dynamic graphs with evolving environments remains largely unexplored, which can be attributed primarily to the intricacies involved in the analysis of both spatial and temporal invariant patterns.
\change{\modelname~performs sequential environment inference via an environment sequential VAE and conducts fine-grained causal interventions to isolate spatio-temporal invariant patterns, accompanied by theoretical guarantees on fulfilling invariance and bounding OOD error in non-stationary environments.}

\section{Problem Formulation}
\label{sec:formulation}
A \textbf{dynamic graph} across a set
of discrete time steps is defined as a sequence of discrete graph snapshots observed over time $\mathcal{G} = \left\{ G^t \right\}_{t=1}^{T}$, where $T$ denotes the time length, $G^{t} = \left(\mathcal{V}^{t}, \mathcal{E}^{t}\right)$ denotes the snapshot at $t$, $\mathcal{V}^{t}$ denotes the node set and $\mathcal{E}^{t} $ denotes the edge set. 
Let $\mathbf{A}^{t} \in \{0,1\}^{N \times N}$ and $\mathbf{X}^{t} \in \mathbb{R}^{N \times d}$ be the adjacency matrix and node feature matrix in timestamp $t$, where $N = | \mathcal{V}^t |$ denotes the number of nodes and $d$ denotes the feature dimensionality. 

\textbf{Future link prediction} task aims to predict the presence of edges in upcoming time periods based on historical graphs, \ie, given $G^{1:T}$, the task aims to learn a predictor $f_{{\boldsymbol{\theta}}}: \mathcal{V} \times \mathcal{V} \mapsto \{0,1\}$ for links at time $T+1$. 

\textbf{Node Classification} task focuses on predicting the labels of nodes in forthcoming time steps based on historical graph data, \ie, given $G^{1:T}$, the task aims to learn a classifier $f_{{\boldsymbol{\theta}}}: \mathcal{V} \mapsto \{1,\cdots, C\}$ for nodes at time $T+1$.

\textbf{OOD on dynamic graphs} refers to the presence of distribution shifts, which occur when $p_\mathrm{test}(G^{1:T}, Y^{T}) \neq p_\mathrm{train}(G^{1:T}, Y^{T})$, \change{where $Y^{T}$ denotes the label at time $T$}. 
It is widely recognized that distribution shifts in graph data are induced by the latent environments~\cite{wu2022handling, rojas2018invariant, buhlmann2020invariance, gong2016domain, arjovsky2019invariant}, which influence the generation of graph data. 
Following~\cite{sinha2017certifying,chen2022bagnn}, an environment variable $\mathbf{e}$ is introduced as a potential unknown indicator for the dynamic graph, drawn from specific distributions.
In this paper, we aim to learn a robust and generalized model $f_{{\boldsymbol{\theta}}}$ to precisely predict the links or class labels in the unseen future environment. 
The optimization objective can be reformulated as: 
\begin{equation}\label{eq:ood}
    \min_{{\boldsymbol{\theta}}} \max_{\mathbf{e} \in \mathbf{E}} \mathbb{E}_{(G^{1:T}, Y^{T}) \sim p(G^{1:T}, Y^{T} | \mathbf{e})} \left [ \ell\left(f_{{\boldsymbol{\theta}}}\left(G^{1:T}\right ), Y^{T}\right ) | \mathbf{e}\right],
\end{equation}
where $\mathbf{E}$ represents the set of environments and the min-max strategy seeks to minimize the empirical risk under the most challenging environment. 
However, it is not practical to optimize Eq.~\eqref{eq:ood} directly since the environment $\mathbf{e}$ is not observable in the data. 
Besides, the environments in the training set may not cover all the environments in the real-world scenarios, \ie, $\mathbf{E}_{\mathrm{train}} \subseteq \mathbf{E}_{\mathrm{test}}$.

\section{\modelname : Evolving
Graph Learning for OOD Generalization}
\label{sec:method}
This section will introduce a framework named \textbf{\modelname}, to solve the OOD generalization problem for dynamic graphs. 
Fig.~\ref{fig:model} shows the overall architecture of \modelname. 
\modelname~follows a \textbf{environment “Modeling-Inferring-Discriminating-Generalizing” paradigm}. 

\begin{figure*}[t]
    \centering
    \includegraphics[width=1\linewidth]{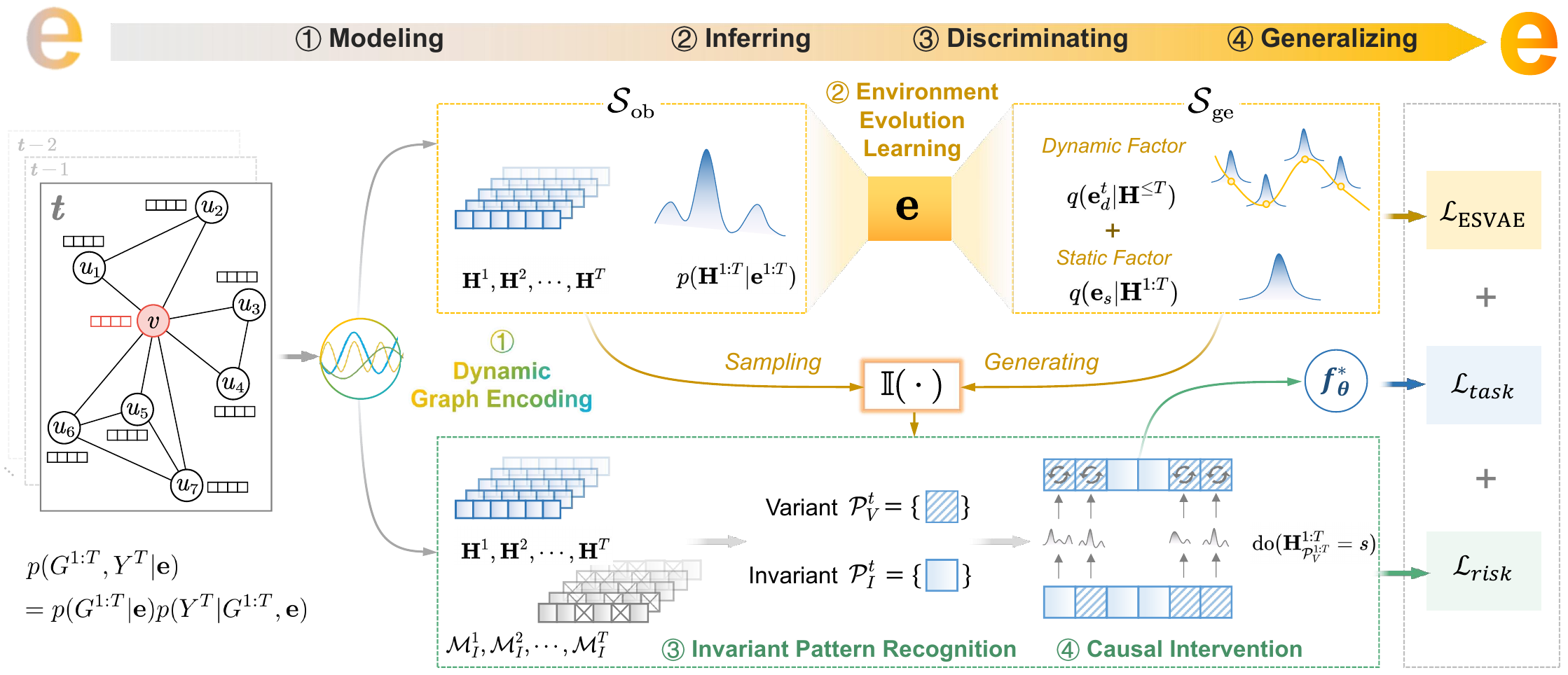}
    \caption{
    The framework of \modelname. \modelname~is following an environment “Modeling-Inferring-Discriminating-Generalizing” paradigm. 
    \change{Step \ding{172}: For a dynamic graph, we first use a spatio-temporal graph attention encoding mechanism to generate node encodings $\{\mathbf{H}^1,\mathbf{H}^2,\cdots,\mathbf{H}^T\}$ and model the latent environments $\mathbf{e}$ (Section~\ref{subsec:graph_encoding}). }
    \change{Step \ding{173}: The environment sequential auto-encoder (ESVAE) learns the environment evolution and then infers the future environment distribution (Section~\ref{subsec:ESVAE}). }
    \change{Step \ding{174}: The environment-aware invariant pattern recognition mechanism $\mathbb{I}(\cdot)$ recognizes the spatio-temporal invariant patterns $\mathcal{P}_I^t$ and variant patterns $\mathcal{P}_V^t$ within the environment for stable predictions (Section~\ref{subsec:invariant_recognition}).} 
    \change{Step \ding{175}: Finally, \modelname~performs node-wise fine-grained causal interventions with sampled and generative environment instances $\mathcal{S}_{\mathrm{ob}}$ and $\mathcal{S}_{\mathrm{ge}}$ to generalize to non-stationary environments (Section~\ref{subsec:intervention}). }
    }
    \label{fig:model}
\end{figure*}

\subsection{Environment Modeling by Dynamic Graph Encoding}
\label{subsec:graph_encoding}
To capture the interaction patterns of the dynamic graph, we first utilize a spatio-temporal graph attention encoding network to make each node attend to its dynamic neighborhood simultaneously. 
We follow the ``spatial first, temporal second'' paradigm~\cite{zhang2022dynamic,li2022autost} to learn time-aware node representations. 
Specifically, we first project the feature $\mathbf{x}_{v}^{t}$ of node $v$ at time $t$: 
\begin{equation}
\label{eq:proj}
    \mathbf{z}_{v}^{t} = \sigma\left(\mathbf{W}^{\top}_1 \left(\mathbf{x}_{v}^{t}\oplus \mathrm{RTE}\left(t\right)\right) + \mathbf{b}\right),
\end{equation}
where $\mathbf{W}_1 \in \mathbb{R}^{d \times d'}$ and $\mathbf{b} \in \mathbb{R}^{d'}$ are learnable parameters, $\mathrm{RTE}(\cdot)$ represents the relative time encoding function, $d'$ indicates the dimensionality of $\mathbf{z}_{v}^{t}$, $\sigma(\cdot)$ is the activation function, and $\oplus$ denotes element-wise addition.
The spatial convolutions are formulated as (the layer superscript is omitted for simplicity):
\begin{align}
    \hat{\mathbf{z}}_{v}^{t} &= \mathbf{z}_{v}^{t} + \sum_{u \in \mathcal{N}^{t}(v)} a_{(u,v)}^{t} \mathbf{z}_{u}^{t}, 
\end{align}
where $\mathcal{N}^{t}(v)$ is node $v$'s neighbors at time $t$, \change{$a_{(u,v)}^{t}$ is the attention score of edge $(u,v)$ at time $t$ calculated by the multi-head attention mechanism~\cite{velivckovic2017graph},} 
and $\hat{\mathbf{z}}_{v}^{t}$ is the updated node representation. 
Subsequently, we can perform temporal convolutions comprehensively for graph snapshots at time $t$ and all those preceding:
\begin{equation}\label{eq:repre}
    \mathbf{h}_{v}^{t} = \frac{1}{t} \sum_{\tau =1}^{t}\hat{\mathbf{z}}_{v}^{\tau}.
\end{equation}
where $\mathbf{h}_{v}^{t}$ denotes the representation of node $v$ at time $t$. 
We can also extend the temporal convolutions to other sequential convolutions or compound them with attention mechanisms. 
Here we use the widely-used attention as in~\cite{velivckovic2017graph}. 
Then we can stack $L$ layers to generate the node representations. 
In this way, each $\mathbf{h}_{v}^{t}$ not only contains information about the current time but also summarizes the co-evolution trend of the neighbors from the recent past to the near future. 

\subsection{Environment Inferring by Evolution Learning}
\label{subsec:ESVAE}
To infer the distribution of evolving environments, we propose an environment sequential variational autoencoder (ESVAE), which can learn from the environment instances given graph snapshots $G^{1:T}$ and infer the incoming environment pattern. 

It can be considered that the link occurrence is decided by the underlying environment $\mathbf{e}_t$ at $t$, which can be revealed by the graph features and structure. 
To this end, the learned node embeddings $\mathbf{H}^t=\{\mathbf{h}_{v}^{t}\}$ can be regarded as environment instances drawn from the ground-truth environment distribution $\mathbf{e}^t$. 
The ESVAE is used to infer the environment distribution $\mathbf{e} \sim q(\mathbf{e}|\mathbf{H}^{1:T})$ from $\{\mathbf{H}^1,\mathbf{H}^2,\cdots,\mathbf{H}^T\}$. 
To capture the evolving pattern of the non-stationary environments, we disentangled the environment representation in the latent space into a static factor $\mathbf{e}_s$ and a dynamic factor $\mathbf{e}^{1:T}_d$. 

\textbf{Priors. }
Specifically, the prior of $\mathbf{e}_s$ can be defined as a standard Gaussian distribution: $\mathbf{e}_s\sim\mathcal{N}(0,1)$. 
The dynamic latent variables $\mathbf{e}^{1:T}_d$ follow a sequential prior
\begin{equation}
    \mathbf{e}^{t}_d|\mathbf{e}^{<t}_d\sim\mathcal{N}\left(\mu_t, \mathbf{D}_{diag}\left(\sigma^2_t\right)\right),
\end{equation}
where $\left[\mu_t, \sigma^2_t\right]=\phi^{prior}_R(\mathbf{e}_{<t}^d)$, $\mu_t$ and $\sigma^2_t$ are the parameters conditioned on all previous dynamic latent variables $\mathbf{e}^{<t}_d$, and $\mathbf{D}_{diag}$ is the diagonal matrix. 
The model $\phi^{prior}_R$ can be parameterized by a recurrent network by updating the hidden states temporally. 
We use LSTM~\cite{hochreiter1997long} as the $\phi^{prior}_R$ in this paper. 
Then we can factorize the prior of $\mathbf{e}$ as: 
\begin{equation}
    p(\mathbf{e})=p(\mathbf{e}_s)p\left(\mathbf{e}^{1:T}_d\right)=p(\mathbf{e}_s)\prod_{t=1}^{T} p\left(\mathbf{e}^{t}_d|\mathbf{e}^{<t}_d\right).
\end{equation}

\textbf{Generation. }
The generation distribution of the environment in timestep $t$ is conditioned on $\mathbf{e}_s$ and $\mathbf{e}^{t}_d$:
\begin{equation}
    \mathbf{H}^t|\mathbf{e}_s,\mathbf{e}^{t}_d\sim\mathcal{N}\left(\mu_{G,t}, diag\left(\sigma^2_{G,t}\right)\right),
\end{equation}
where $\left[\mu_{G,t}, diag\left(\sigma^2_{G,t}\right)\right]=\phi^{Dec}\left(\mathbf{e}_s,\mathbf{e}^{t}_d\right)$ and the decoder $\phi^{Dec}$ can be a highly flexible function, and we use a 2-layer multilayer perceptron here. 
The overall generative process can be formulated by the factorization: 
\begin{equation}
    p\left(\mathbf{H}^{1:T},\mathbf{e}_s,\mathbf{e}^{1:T}_d\right)=p\left(\mathbf{e}_s\right)\prod_{t=1}^{T}p\left(\mathbf{H}^t|\mathbf{e}_s,\mathbf{e}^{t}_d\right) p\left(\mathbf{e}^{t}_d|\mathbf{e}^{<t}_d\right).
\end{equation}

\textbf{Inference. }
The posterior distributions can be approximated using variational inference by the ESVAE:
\change{
\begin{equation}
    \mathbf{e}_s\sim\mathcal{N}\left(\mu_s, diag\left(\sigma^2_s\right)\right), \mathbf{e}_d\sim\mathcal{N}\left(\mu_t, diag\left(\sigma^2_t\right)\right),
\end{equation}
}
where $\left[\mu_s, diag\left(\sigma^2_s\right)\right]=\phi^{Enc}_s\left(\mathbf{H}^{1:T}\right)$ and $\left[\mu_t, diag\left(\sigma^2_t\right)\right]=\phi^{Enc}_d\left(\mathbf{H}^{\le t}\right)$. 
The static factor $\mathbf{e}_s$ is conditioned on the whole sequence, while the dynamic factor $\mathbf{e}_d$ is inferred by a recurrent encoder $\phi^{Enc}_d$ and only conditioned on graphs of the previous timestamps. 
The environment inference model can be factorized as:
\begin{equation}
    q\left(\mathbf{e}_d^{1:T},\mathbf{e}_s|\mathbf{H}^{1:T}\right)=q\left(\mathbf{e}_s|\mathbf{H}^{1:T}\right)\prod_{t=1}^Tq\left(\mathbf{e}_d^t|\mathbf{H}^{\le t}\right).
\end{equation}

\textbf{Optimization. }
The objective of ESVAE can be formulated as the negative variational lower bound, computed timestep-wise:
\begin{equation}
\begin{aligned}
      \mathcal{L}_{\rm SVAE}= &\mathbb{E}_{q\left(\mathbf{e}_s,\mathbf{e}_d^{1:T}|\mathbf{H}^{1:T}\right)}\left[-\sum_{t=1}^T\log p\left(\mathbf{H}^t|\mathbf{e}_s,\mathbf{e}_d^t\right)\right]\\
      &+\mathcal{D}_{\rm KL}\left(q\left(\mathbf{e}_s||\mathbf{H}^{1:T}\right)||p\left (\mathbf{e}_s\right)\right)\\
      &+\sum_{t=1}^T\mathcal{D}_{\rm KL}\left(q\left(\mathbf{e}_d^t|\mathbf{H}^{\le t}\right)||p\left(\mathbf{e}_d^t|\mathbf{e}_d^{<t}\right)\right).
\end{aligned}
\end{equation}

To encourage the static factor $\mathbf{e}_s$ to exclude dynamic information, we aim for $\mathbf{e}_s$ to remain relatively stable even when the dynamic information undergoes significant changes.
In an ideal scenario, the static factor $\mathbf{e}_s$ learned from the node representation sequence should be similar to the static factor learned from its corresponding shuffled sequence.
To achieve this, we treat the shuffled sequence as a positive sample of the static factor and select another sequence at random as the negative sample.
A triplet loss $\mathcal{L}_{s}$ is introduced utilizing a triplet of static factors:
\begin{equation}
    \mathcal{L}_{s}=\max\left ( D\left(\mathbf{e}_s,\mathbf{e}^{pos}_s\right )-D\left(\mathbf{e}_s,\mathbf{e}^{neg}_s\right )+m,0\right ),
\end{equation}
where $D(\cdot)$ is the Euclidean distance function and $m$ is the margin. 

To encourage the dynamic $\mathbf{e}_d$ to carry adequate dynamic information of evolution, we design an auxiliary task to utilize dynamic information-related signals as the regularization imposed on $\mathbf{e}_d$. 
The dynamic representation $\mathbf{e}_d$ can be learned by forcing it to predict the dynamics of graphs. 
To this end, we first use the K-means algorithm~\cite{hartigan1979algorithm} to split all nodes $v\in \mathcal{V}$ into $m$ clusters $\{\mathcal{V}_1, \mathcal{V}_2, \cdots, \mathcal{V}_m\}$ and use the structural changes of these clusters to represent the dynamics of the whole graph. 
Here we set $m=10$. 
To this end, we use the 1-order structural entropy~\cite{li2016structural} to estimate the uncertainty of every node for every timestamp. 
\begin{equation}
    H^1(G^t)=-\sum_{v\in \mathcal{V}}\frac{d^t_{v}}{vol(G^t)}\log_2\frac{d^t_{v}}{vol(G^t)},
\end{equation}
where $d^t_v$ is node $v$'s degree in $t$ and $vol(G^t)$ is the sum of the degree of all nodes $\mathcal{V}$ in $G^t$. 
The one-dimensional $H^1(G^t)$ measures the uncertainty of $\mathcal{V}$ in $G^t$ and the uncertainty of node $v$ is denoted as $SE^t_v$. 
Then the average of uncertainty magnitudes for every node cluster can be computed to represent the structural features. 
We set the indices of clusters with the top-k largest magnitudes at $t$ as the pseudo label $Y_{pse}^t$ for prediction. 
\begin{equation}
    SE^t_{\mathcal{V}_i}=\frac{1}{|\mathcal{V}_i|}\sum_{v \in \mathcal{V}_i}SE^t_v, \mathcal{V}_i\in\{\mathcal{V}_1, \mathcal{V}_2, \cdots, \mathcal{V}_m\}.
\end{equation}
The objective function can be formulated as: 
\begin{equation}
    \mathcal{L}_{d}=CE(\phi_{SE}(\mathbf{e}_d^t),Y_{pse}^t),
\end{equation}
where $CE(\cdot)$ denotes the cross-entropy function and $\phi_{SE}(\cdot)$ denotes the cluster predictor. 

Then the overall objective of ESVAE is:
\begin{equation}
\label{eq:loss_ESVAE}
    \mathcal{L}_{\mathrm{ESVAE}} =  \mathcal{L}_{\mathrm{SVAE}}+ \alpha_1\mathcal{L}_{s}+ \alpha_2\mathcal{L}_{d},
\end{equation}
where $\alpha_1$ and $\alpha_2$ are the loss weight parameters.

\subsection{Environment Discriminating by Invariant Pattern Recognition}
\label{subsec:IPR}
We intend to identify spatio-temporal invariant patterns by utilizing the inferred environmental distributions, which adhere to the Independent Causal Mechanism (ICM) assumption~\cite{pearl2009causal, pearl2010causal, peters2017elements}. 

\textbf{Invariance Assumption.} 
\change{To encourage our \modelname~to rely on the invariant correlations within node representations and target labels that can generalize under distribution shifts, 
we propose uncovering the node-wise invariant pattern by introducing a pair of complementary feature masks, denoted as $\mathcal{M}_{I}$ and $\mathcal{M}_{V}$.}
\change{Intuitively, $\mathcal{M}_{I}$ acts as a selector that highlights the dimensions of node representations $\mathbf{H}^t$ which correspond to invariant patterns whose correlations that remain stable across temporal dynamics and distribution shifts. 
In contrast, $\mathcal{M}_{V}$ plays the complementary role, filtering out the dimensions that are more variant, which may fluctuate with time or environment changes.}
\change{Formally, the invariant representation is obtained as $\mathbf{H}_I^t = \mathcal{M}_I^t \odot \mathbf{H}^t$, and $\mathbf{H}_V^t$ is complementary to $\mathbf{H}_I^t$.}
\change{The masks are optimized jointly with the model to automatically partition each node embedding into invariant and variant subspaces. This design ensure that predictions mainly rely on $\mathbf{H}_I^t$, which is less sensitive to distribution shifts.}
We can make Assumption~\ref{asm:invariance}.
\begin{asm}
\label{asm:invariance}
    Given a dynamic graph $G^{1:T}$, there exist invariant patterns $\mathcal{P}_I^t$ and variant patterns $\mathcal{P}_V^t$ in each node representation that lead to generalized out-of-distribution prediction. Node representation $\mathbf{H}_I^t$ under $\mathcal{P}_I^t$ uncovered by $\mathcal{M}_{I}$ should satisfy:
    \begin{itemize}[leftmargin=*]
        \item Invariance Property: $\forall \mathbf{e} \in \mathbf{E}$, $\left(\mathcal{M}_{I}^t \odot \mathbf{H}^t\right) \models \mathcal{P}_I^t$, s.t. $p\left(Y^{t}|\mathcal{P}_I^t, \mathbf{e}^t\right) = p\left(Y^{t}|\mathcal{P}_I^t\right)$;
        \item Sufficient Condition: $Y^{t} = f\left(\mathbf{H}_I^t\right)+\epsilon$, i.e., $Y^t \perp \mathcal{P}_V^t \mid \mathcal{P}_I^t$, where $f\left(\cdot\right)$ is the node classifier or link predictor, and $\epsilon$ is the independent noise.
    \end{itemize}
    where $\models$ denotes the logic structure $\mathcal{M}_{I}^t \odot \mathbf{H}^t$ models the pattern $\mathcal{P}_I^t$. $\perp$ and $|$ represent the variable $Y^t$ that is independent of $\mathcal{P}_V^t$ but only relies on $\mathcal{P}_I^t$. 
\end{asm}

Assumption~\ref{asm:invariance} suggests that the node representations $\mathbf{H}_I^t$, when governed by $\mathcal{P}_I^t$, include invariant causal features. These features are instrumental in facilitating generalized OOD predictions across different time periods. The Invariance Property ensures the $\mathcal{M}_{I}^t$ mask can always recognize the time-invariant pattern $\mathcal{P}_I^t$. The Sufficient Condition ensures that the attributes within $\mathbf{H}_I^t$ are sufficiently robust for the predictor to accurately forecast outcomes.

\textbf{Invariant Pattern Recognition.} To learn the node-wise invariant patterns, we propose the implementation of the learnable complementary dimensional masks $\mathcal{M}_{I}$ and $\mathcal{M}_{V}$ by Proposition~\ref{prop:mask}.

\begin{prop}
\label{prop:mask}
    Given a series of node representations $\mathbf{h}_v^{1:t} \in \mathbb{R}^{t\times d}$, we initialize the learnable invariant dimensional mask $\mathcal{M}_{I}^t(v)$ by:
    \begin{equation}
    \label{eq:mask_invariant}
        \mathcal{M}_{I}^t(v) = \left[\mathbb{I}\left(\mathrm{Var}\left(\mathbf{h}_v^{1:t} > \delta \right)\right )\right] \odot \mathbf{W}_{I},
    \end{equation}
    where $\mathbb{I}(\cdot)$ is the bit-wise indicator function, $\mathrm{Var}(\cdot)$ calculates bit variance within the past $t$ times, $\delta$ is the decision threshold, and $\mathbf{W}_{I}$ is the learnable parameters. Then, $\mathcal{M}_{V}^t(v)$ can be calculated as:
    \begin{equation}
    \label{eq:mask_variant}
        \mathcal{M}_{V}^t(v) = \mathbf{1} - \mathcal{M}_{I}^t(v).
    \end{equation}
\end{prop}

Proposition~\ref{prop:mask} offers a practical approach for initializing the pair of masks, with the variability in the bits indicating the current state of invariance and serving as an optimal starting point for optimization.
It is important to recognize that Proposition~\ref{prop:mask} alone does not entirely ensure the global accuracy of mask identification. However, it should be optimized in conjunction with the $\mathcal{L}_{\mathrm{risk}}$ loss, which provides a theoretical guarantee.
Supported by Assumption \ref{asm:invariance} and Proposition \ref{prop:mask}, the spatio-temporal invariant patterns $\mathcal{P}_I^t(v)$ and variant patterns $\mathcal{P}_V^t(v)$ can be distinguished within evolving environments for each node by applying $\mathcal{M}_{I}^t(v)$ and $\mathcal{M}_{V}^t(v)$, respectively.

\subsection{Environment Generalizing by Causal Intervention}
\label{subsec:intervention}

Based on the inferred environment distributions from Section~\ref{subsec:ESVAE} and the spatio-temporal invariant patterns identified in Section~\ref{subsec:IPR}, we introduce a fine-grained intervention mechanism that is applied at the node level for causal inference. This mechanism is designed to help the model distinguish between spurious and causal correlations.

It is time-consuming and intractable to directly intervene in variant patterns node-wise. 
Therefore, we propose an approximate method for fine-grained spatio-temporal interventions using a combination of observed samples $\mathcal{S}_{\mathrm{ob}}$ and generated samples $\mathcal{S}_{\mathrm{ge}}$ by sampling and replacing variant patterns $\mathcal{P}^t_I(v)$. 
The observed samples can be formulated as: 
\begin{equation}
\label{eq:sob}
    \bigcup_{v \in \mathcal{V}} \bigcup_{t=1}^{T} \{\mathbf{h}_{v}^{t}\} \in \mathbb{R}^{N\times T\times d'} \xlongequal[]{\rm def}  \mathcal{S}_{\mathrm{ob}}.  
\end{equation} 
By directly sampling the latent variable $\mathbf{e}$ from its distribution $p(\mathbf{e})$ for $\mathcal{S}_{\mathrm{ge}}$, we can synthesize samples from the conditional distribution $p(\mathbf{h}|\mathbf{e})$, thereby substantially enriching the variety of the generated samples.
\begin{equation}
    \bigcup_{v \in \mathcal{V}} \bigcup_{t=1}^{T} \{\mathbf{h}_{v}^{t}\sim p(\mathbf{h}|\mathbf{e}^t)|\mathbf{e}^t\sim p(\mathbf{e})\} \xlongequal[]{\rm def}  \mathcal{S}_{\mathrm{ge}}. 
\end{equation}

For representations $\mathbf{h}_{v}^{t}$ of node $v$ at time $t$, $\mathcal{P}^t_I(v)$ and $\mathcal{P}^t_V(v)$ can be obtained by applying $\mathbb{I}(\cdot)$. 
As $\mathcal{P}^t_V(v)$ represents dimensionality indices related with variant patterns in $\mathbf{h}_{v}^{t}$, we generate the intervention set $\mathbf{s}_v$ by randomly sampling instances and then perform node-wise replacing: 
\begin{equation}
\label{eq:intervention}
    \mathbf{h}_{v}^{\mathbf{e},t}\left[\mathcal{P}^t_V(v)\right] := \mathbf{s}_v, \quad \mathbf{s}_{v} = \left \{ s | s \in \mathcal{S}_{\mathrm{ob}} \cup \mathcal{S}_{\mathrm{ge}} \right \}.
\end{equation}
Let $\mathcal{S}=\mathcal{S}_{\mathrm{ob}}\cup \mathcal{S}_{\mathrm{ge}}$ be the observed and generated environment instances for interventions. 
The causal intervention loss can be formulated as: 
\begin{equation}
\label{eq:loss_risk}
\begin{aligned}
    \mathcal{L}_{\mathrm{risk}} = {\rm Var}&\left\{ 
    \mathbb{E}
    \left [ \ell\left(f_{{\boldsymbol{\theta}}}\left(G^{1:T}\right), Y^{T} | \operatorname{do}(\mathbf{H}^{1:T}_{\mathcal{P}^{1:T}_V} = s)\right)\right],\right.\\
    &\left.s\in \mathcal{S},\mathbf{e}\sim q(\mathbf{e}),(G^{1:T}, Y^{T}) \sim p(G^{1:T}, Y^{T} | \mathbf{e})  \right\},     
\end{aligned}
\end{equation}
where $q(\mathbf{e})$ represents the distribution of $\mathbf{e}$ discussed in Section~\ref{subsec:ESVAE}, while $\operatorname{do}(\cdot)$ denotes the $do$-calculus, which is employed to intervene and manipulate the variant patterns.

Combining the future link prediction task loss, the ESVAE loss in Eq.~\eqref{eq:loss_ESVAE} and the casual intervention loss in Eq.~\eqref{eq:loss_risk}, the overall loss of \modelname~is:
\begin{equation}
\label{eq:final_loss}
    \mathcal{L} = \mathcal{L}_{\mathrm{task}} + \beta_1\mathcal{L}_{\mathrm{risk}} + \beta_2 {\mathcal{L}}_{\mathrm{ESVAE}},
\end{equation}
\begin{equation}
\label{eq:loss_task}
    \mathcal{L}_{\mathrm{task}} = \mathbb{E}_{\mathbf{e}\sim q_\phi(\mathbf{e}), (G^{1:T}, Y^{T}) \sim p(G^{1:T}, Y^{T} | \mathbf{e})} \left [ \ell\left(g(\mathbf{H}^{1:T}_{\mathcal{P}^{1:T}_I}), Y^{T}\right)\right],
\end{equation}
where $\beta_1$ and $\beta_2$ are two hyperparameters that control the contribution of the $\mathcal{L}_{\mathrm{risk}}$ and $\mathcal{L}_{\mathrm{ESVAE}}$, respectively. 
We can have the following two propositions: 
\begin{prop}
\label{prop:encourage}
    Minimizing Eq.~\eqref{eq:final_loss} encourages \modelname~to fulfill the Invariance Property and the Sufficient Condition as outlined in Assumption~\ref{asm:invariance}.
\end{prop}
\begin{proof}
To prove Proposition~\ref{prop:encourage}, we generally reduce to prove minimizing $\mathcal{L}_{\rm}+\beta_1 \mathcal{L}_{\rm{risk}}$ satisfies Proposition~\ref{prop:encourage}.

We begin by presenting the following lemma, which allows us to reformulate Assumption~\ref{asm:invariance} from an information theory perspective.

\begin{Lemma}\label{Lemma:1}
Assumption~\ref{asm:invariance} posits that the Invariance Property and Sufficient Condition can be respectively recast in terms of mutual information:
\begin{itemize}[leftmargin=*]
    \item Invariance Property: $p(Y^{t}|\mathcal{P}_I^t, \mathbf{e}^t) = p(Y^{t}|\mathcal{P}_I^t) \Leftrightarrow   \mathcal{I} (Y^{t};\mathbf{e}^t| \mathcal{P}_I^t) = 0$;
    \item Sufficient Condition: $Y^t \perp \mathcal{P}_I^t \mid \mathcal{P}_V^t \Leftrightarrow  \mathcal{I}(Y^{t};\mathcal{P}_I^t)$ is maximized.
\end{itemize}
\end{Lemma}

Lemma~\ref{Lemma:1} can be easily proved following the contradiction proofs proposed in \cite{wu2022handling}. Next, we prove Proposition~\ref{prop:encourage} based on Lemma~\ref{Lemma:1}.

We initially demonstrate that minimizing the expected loss term ($\mathcal{L}_\mathrm{task}$) in Eq.~\eqref{eq:final_loss} promotes \modelname~to meet the Sufficient Condition of Assumption~\ref{asm:invariance}. It is straightforward to see that maximizing $\mathcal{I}(Y^{t};\mathbf{H}^{1:t})$ is equivalent to minimizing $\mathcal{I}(Y^{t}; G^{1:t}|\mathbf{H}^{1:t})$ under the conditional distribution $q(\mathbf{H}^{1:t}|G^{1:t})$.
By considering $q(\mathbf{H}^{1:t}|G^{1:t})$ as the variational distribution, the resulting upper bound is given by:

\begin{align}
    &\mathcal{I}(Y^{t}; G^{1:t}|\mathbf{H}^{1:t}) \\
    &= \mathcal{D}_{\mathrm{KL}} \left [ p(Y^{t} | G^{1:t}, \mathbf{e}^{1:t})~\big \|~p(Y^{t} | \mathbf{H}^{1:t},\mathbf{e}^{1:t}) \right ] \notag \notag \\
    &= \mathcal{D}_{\mathrm{KL}} \left [ p(Y^{t} | G^{1:t}, \mathbf{e}^{1:t})~\big \|~q(Y^{t}|\mathbf{H}^{1:t})\right ]
    \notag \notag \\&\quad- \mathcal{D}_{\mathrm{KL}} \left [p(Y^{t}|\mathbf{H}^{1:t},\mathbf{e}^{1:t})~\big \|~q(Y^{t}|\mathbf{H}^{1:t}) \right ] \notag \notag \\
    &\le \mathcal{D}_{\mathrm{KL}} \left [ p(Y^{t} | G^{1:t}, \mathbf{e}^{1:t})~\big \|~q(Y^{t}|\mathbf{H}^{1:t})\right ]\notag \notag \\
    &\le \min_{q(Y^{t}|\mathbf{H}^{1:t})} \mathcal{D}_{\mathrm{KL}} \left [ p(Y^{t} | G^{1:t}, \mathbf{e}^{1:t})~\big \|~q(Y^{t}|\mathbf{H}^{1:t})\right ].
\end{align}

Derived from Jensen's Inequality~\cite{jensen1906fonctions}:
\begin{align}
\label{Eq:jensen}
    &\mathcal{D}_{\mathrm{KL}} \left [ p(Y^{t} | G^{1:t}, \mathbf{e}^{1:t})~\big \|~q(Y^{t}|\mathbf{H}^{1:t})\right ]\notag \notag \\
    &=\mathbb{E}_{\mathbf{e},p(G^{1:t}, {Y}^{t}),\mathbf{H}^{1:t}} \left [  \log \frac{p\left(Y^{t}|G^{1:t},\mathbf{e}^{1:t}\right)}{q\left(Y^{t} |\mathbf{H}^{1:t}\right)} \right ] \notag \notag \\
    &\le \mathbb{E}_{\mathbf{e},p({G}^{1:t}, {Y}^{t}) }\left [ \log \frac{p(Y^{t} |G^{1:t},\mathbf{e}^{1:t})}{\mathbb{E}_{\mathbf{H}^{1:t}}\left [q(Y^{t} |\mathbf{H}^{1:t})\right ]} \right ].
\end{align}

As Section~\ref{subsec:graph_encoding} guarantees $q(\mathbf{H}^{1:t}|G^{1:t})$ a $\delta$-distribution, we have:
\begin{align}
    &\min_{q(Y^{t}|\mathbf{H}^{1:t})} \mathcal{D}_{\mathrm{KL}} \left [ p(Y^{t} | G^{1:t}, \mathbf{e}^{1:t})~\big \|~q(Y^{t}|\mathbf{H}^{1:t})\right ]  \notag \notag \\
    \Leftrightarrow  &\min_{\boldsymbol{\theta}}\mathbb{E}_{\mathbf{e}\sim q_\phi(\mathbf{e}^{1:t}), p(G^{1:t}, {Y}^{t})}  \left [ \ell\left(f(\mathbf{H}_I^t), {Y}^{t}\right)\right ].
\end{align}

Next, we demonstrate that minimizing the variance term ($\mathcal{L}_\mathrm{risk}$) in Eq.~\eqref{eq:final_loss} promotes \modelname~to fulfill the Invariance Property. 
We have:
\begin{align}
    &\mathcal{I}(Y^{t};\mathbf{e}^{1:t}|\mathbf{H}^{1:t}) \\&= \mathcal{D}_{\mathrm{KL}}\left [p(Y^{t}|\mathbf{H}^{1:t},\mathbf{e}^{1:t})~\big \|~p(Y^{t}|\mathbf{H}^{1:t})\right ] \notag \notag \\
    &= \mathcal{D}_{\mathrm{KL}}\left [p(Y^{t}|\mathbf{H}^{1:t},\mathbf{e}^{1:t})~\big \|~\mathbb{E}_{\mathbf{e}}\left [p(Y^{t}|\mathbf{H}^{1:t},\mathbf{e}^{1:t})\right ]\right ]\notag \notag \\
    &= \mathcal{D}_{\mathrm{KL}} \left [q(Y^{t}|\mathbf{H}^{1:t})~\big \|~\mathbb{E}_{\mathbf{e}} q(Y^{t}|\mathbf{H}^{1:t})\right ]\notag \notag \\
    &\;\;\;\;-\mathcal{D}_{\mathrm{KL}} \left [q(Y^{t}|\mathbf{H}^{1:t})~\big \|~p(Y^{t}|\mathbf{H}^{1:t},\mathbf{e}^{1:t})\right ]\notag \notag \\
    &\;\;\;\;- \mathcal{D}_{\mathrm{KL}}\left [\mathbb{E}_{\mathbf{e}}\left [p(Y^{t}|\mathbf{H}^{1:t},\mathbf{e}^{1:t})\right ]~\big \|~\mathbb{E}_{\mathbf{e}}\left [q(Y^{t}|\mathbf{H}^{1:t})\right ]\right ]\notag \notag \\
    &\le \mathcal{D}_{\mathrm{KL}}  \left [q(Y^{t}|\mathbf{H}^{1:t})~\big \|~\mathbb{E}_{\mathbf{e}} q(Y^{t}|\mathbf{H}^{1:t})  \right ] \notag \notag \\
    & \le \min_{q(Y^{t}|\mathbf{H}^{1:t})} \mathcal{D}_{\mathrm{KL}}  \left [q(Y^{t}|\mathbf{H}^{1:t})~\big \|~\mathbb{E}_{\mathbf{e}} q(Y^{t}|\mathbf{H}^{1:t})  \right ].
\end{align}

Similarly, we reach:
\begin{align}
    &\mathcal{D}_{\mathrm{KL}} \left [q(Y^{t}|\mathbf{H}^{1:t})~\big \|~\mathbb{E}_{\mathbf{e}} q(Y^{t}|\mathbf{H}^{1:t})\right ]\notag \notag \\
    &=\mathbb{E}_{\mathbf{e},p(G^{1:t}, {Y}^{t}) }\mathbb{E}_{\mathbf{H}^{1:t}}  \left [  \log \frac{q(Y^{t} |\mathbf{H}^{1:t})}{\mathbb{E_{\mathbf{e}}}q(Y^{t} |\mathbf{H}^{1:t})}  \right ].
\end{align}

Derived from Jensen's Inequality:
\begin{align}
    & \mathcal{D}_{\mathrm{KL}} \left [q(Y^{t}|\mathbf{H}^{1:t})~\big \|~\mathbb{E}_{\mathbf{e}} q(Y^{t}|\mathbf{H}^{1:t})\right ] \notag \notag \\
    & \le \mathbb{E}_\mathbf{e} \left [ \left | \ell\left(f_{{\boldsymbol{\theta}}}\left(G^{1:t}\right ), {Y}^{t} \right ) - \mathbb{E}_\mathbf{e} \left [\ell\left(f_{{\boldsymbol{\theta}}}\left(G^{1:t}\right), {Y}^{t} \right ) \right ] \right | \right].
\end{align}

Finally, we reach:
\begin{align}
    &\min_{q(Y^{t}|\mathbf{H}^{1:t})} \mathcal{D}_{\mathrm{KL}} \left [q(Y^{t}|\mathbf{H}^{1:t})~\big \|~\mathbb{E}_{\mathbf{e}} q(Y^{t}|\mathbf{H}^{1:t})\right ]  \notag \notag \\
    &\Leftrightarrow  \min_{\boldsymbol{\theta}}{\rm Var}\left\{ 
    \mathbb{E} \left [\ell\left(f_{{\boldsymbol{\theta}}}\left(G^{1:t}\right), Y^{t} | \operatorname{do}(\mathbf{H}^{1:t}_{V} = s)\right)\right], \right.\notag \\
    &\left.s\in \mathcal{S},\mathbf{e}\sim q(\mathbf{e}^{1:t}),(G^{1:t}, Y^{t}) \sim p(G^{1:t}, Y^{t} | \mathbf{e}^{1:t})  \right\}.
\end{align}

We have completed the proof for Proposition~\ref{prop:encourage}.
\end{proof}
\change{\textbf{Key Takeaway:} Proposition~\ref{prop:encourage} shows that minimizing the proposed objective encourages \modelname~to rely only on invariant causal patterns while filtering out spurious ones. In other words, the loss function guides the model to focus on correlations that remain stable across environments, ensuring predictions are robust under distribution shifts.}

\begin{prop}
\label{prop:bound}
    Optimizing Eq.~\eqref{eq:final_loss} is equivalent to minimizing the upper bound of the OOD generalization error in Eq.~\eqref{eq:ood}.
\end{prop}
\begin{proof}
To prove Proposition~\ref{prop:bound}, we generally reduce to prove minimizing $\mathcal{L}_{\rm}+\beta_1 \mathcal{L}_{\rm{risk}}$ satisfies Proposition~\ref{prop:bound}. The OOD generalization error can be measured by the KL-divergence:
\begin{align}
\label{Eq:error}
    &\mathcal{D}_{\mathrm{KL}} \left [ p(Y^{t}|G^{1:t},\mathbf{e}^{1:t})~\Big \|~q(Y^{t}|G^{1:t}) \right ] \notag \\
    &= \mathbb{E}_{\mathbf{e},p(G^{1:t}, Y^{t}) } \left [  \log \frac{p(Y^{t}|G^{1:t},\mathbf{e}^{1:t})}{q(Y^{t}|G^{1:t})} \right].
\end{align}

Inspired by~\cite{federici2021information, wu2022handling}, we propose the following lemma for rewriting the OOD generalization error.

\begin{Lemma}\label{Lemma:error}
The OOD generalization error is upper bounded by:
\begin{align}
    &\mathcal{D}_{\mathrm{KL}} \left [ p(Y^{t}|G^{1:t},\mathbf{e}^{1:t})~\Big \|~q(Y^{t}|G^{1:t}) \right ] \notag \\
    &\le  \mathcal{D}_{\mathrm{KL}} \left [ p(Y^{t} | G^{1:t}, \mathbf{e}^{1:t})~\Big \|~q(Y^{t}|\mathbf{H}^{1:t})\right ],
\end{align}
where $q(Y^{t}|\mathbf{H}^{1:t})$ is the inferred variational distribution. 
\end{Lemma}

\begin{proof}
\ 
\newline
\begin{align}
    &\mathcal{D}_{\mathrm{KL}} \left [ p(Y^{t}|G^{1:t},\mathbf{e}^{1:t})~\Big \|~q(Y^{t}|G^{1:t}) \right ] \notag \\
    &= {\mathbb{E}_{\mathbf{e}},p(G^{1:t}, Y^{t}) } \left [  \log \frac{p(Y^{t}|G^{1:t} ,\mathbf{e}^{1:t})}{q(Y^{t} |G^{1:t})} \right] \notag \\
    &=\mathbb{E}_{\mathbf{e},p(G^{1:t}, Y^{t}) } \left [  \log \frac{p(Y^{t} |G^{1:t},\mathbf{e}^{1:t})}{\mathbb{E}_{\mathbf{H}^{1:t}}\left [q(Y^{t} |\mathbf{H}^{1:t})\right ]} \right] \notag \\
    & \le \mathbb{E}_{\mathbf{e},p(G^{1:t}, Y^{t}),\mathbf{H}^{1:t} }\left [  \log \frac{p(Y^{t} |G^{1:t},\mathbf{e}^{1:t})}{q(Y^{t} |\mathbf{H}^{1:t})} \right ] \notag \\
    & = \mathcal{D}_{\mathrm{KL}} \left [ p(Y^{t} | G^{1:t}, \mathbf{e}^{1:t})~\Big \|~q(Y^{t}|\mathbf{H}^{1:t})\right ].
\end{align}
\end{proof}

\change{\textbf{Key Takeaway:} Proposition~\ref{prop:bound} establishes that optimizing the objective effectively minimizes an upper bound of the OOD generalization error. Practically, this indicates \modelname~does not merely fit training distributions but is theoretically guaranteed to generalize better to unseen environments.}

Based on Lemma~\ref{Lemma:1}, we adapt the first two terms in Eq.~\eqref{eq:final_loss} as:
\begin{align}
    \min_{q(\mathbf{H}^{1:t}|G^{1:t}),q(Y^{t},\mathbf{H}^{1:t})} &\mathcal{D}_{\mathrm{KL}} \left [ p(Y^{t} | G^{1:t}, \mathbf{e}^{1:t})~\Big \|~q(Y^{t}|\mathbf{H}^{1:t})\right ] \notag \notag \\
    &+ \mathcal{I}(Y^{t};\mathbf{e}^{1:t}|\mathbf{H}^{1:t}).
\end{align}

Based on Lemma~\ref{Lemma:error}, we validate that:
\begin{align}
    &\min_{\boldsymbol{\theta}} \mathcal{L}_\mathrm{task} + \beta_1 \mathcal{L}_\mathrm{risk} \notag \\
    &\Leftrightarrow \min_{q(\mathbf{H}^{1:t}|G^{1:t}),q(Y^{t},\mathbf{H}^{1:t})} \mathcal{D}_{\mathrm{KL}} \left [ p(Y^{t} | G^{1:t}, \mathbf{e}^{1:t})~\Big \|~q(Y^{t}|\mathbf{H}^{1:t})\right ]\notag \\
    &\quad\quad\quad\quad\quad\quad\quad\quad\quad+ \mathcal{I}(Y^{t};\mathbf{e}^{1:t}|\mathbf{H}^{1:t}) \notag \\
    & \ge \min_{q(\mathbf{H}^{1:t}|G^{1:t}),q(Y^{t},\mathbf{H}^{1:t})} \mathcal{D}_{\mathrm{KL}} \left [ p(Y^{t} | G^{1:t}, \mathbf{e}^{1:t})~\Big \|~q(Y^{t}|\mathbf{H}^{1:t})\right ] \notag \\
    & \ge \mathcal{D}_{\mathrm{KL}} \left [ p(Y^{t}|G^{1:t},\mathbf{e}^{1:t})~\Big \|~q(Y^{t}|G^{1:t}) \right ].
\end{align}
We conclude the proof for Proposition~\ref{prop:bound}.
\end{proof}

Proposition~\ref{prop:encourage} is established without relying on strong assumptions, and Proposition~\ref{prop:bound} provides a guarantee for the generalization error after learning. 
In the context of the Structural Causal Model (SCM) applied to dynamic graphs, optimizing Eq.~\eqref{eq:final_loss} serves to mitigate the detrimental effects arising from spurious correlations between labels and variant patterns, while simultaneously reinforcing the invariant causal correlations that persist across various latent environments (as depicted in Figure~\ref{fig:SCM}).

\subsection{Training Process and Computation Complexity}
Algorithm~\ref{alg:training} shows the overall training process of \modelname. 
\begin{algorithm}[!t]
    \caption{Training process of \modelname.}
    \label{alg:training}
    \KwIn{Dynamic graph $\mathcal{G} = (\{ G \}_{t=1}^{T})$ with labels $Y^{1:T}$ of node classes or link occurrence; Training epochs $E$; Intervention times $S$; Hyperparameters $\alpha_1$, $\alpha_2$, $\beta_1$ and $\beta_2$.}
    \KwOut{
    Predicted label $Y^{T+1}$ of node classes or link occurrence at time $T+1$.}
    \BlankLine 
    Parameter Initialization\;
    \For{$i=1,2,\cdots,E$}{
         \tcp{Dynamic Graph Encoding} 
         Acquire node representations at each timestamp $\mathbf{h}_{v}^{t} \gets$ Eq.~\eqref{eq:repre}\;
        \tcp{Environment Evolution Learning} 
        Establish the library of the observed environment samples $\mathcal{S}_{\mathrm{ob}} \gets$ Eq. \eqref{eq:sob}\;
        Infer the environment distribution with $\mathcal{L}_{\mathrm{ESVAE}} \gets$ Eq.~\eqref{eq:loss_ESVAE}\; 
        Establish the generated samples library $\mathcal{S}_{\mathrm{ge}}$\;
        \tcp{Environment-aware Invariant Pattern Recognition} 
        Learn the invariant dimensional mask $\mathcal{M}_{I}^t(v)\gets$ Eq.~\eqref{eq:mask_invariant} and variant dimensional mask $\mathcal{M}_{V}^t(v) \gets$ Eq.~\eqref{eq:mask_variant}\;
        Recognize the invariant patterns $\mathcal{P}^t_{I}(v)$ and variant patterns $\mathcal{P}^t_{V}(v)$ for each node\;
        Calculate the task loss $\mathcal{L}_{\mathrm{task}} \gets$ Eq.~\eqref{eq:loss_task}\;
         \tcp{Causal Intervention for Generalization}
        \For{$j=1,2,\cdots,S$}{
            Randomly select samples from $\mathcal{S}_{\mathrm{ob}} \cup \mathcal{S}_{\mathrm{ge}}$\; 
            Implement interventions for individual nodes as Eq. \eqref{eq:intervention}\;
            Calculate intervention loss $\mathcal{L}_{\mathrm{risk}} \gets$ Eq.~\eqref{eq:loss_risk}\;
        }
        \tcp{Parameter Optimization}
        Update parameters by minimizing $\mathcal{L} \gets$ Eq.~\eqref{eq:final_loss}.
    }
\end{algorithm}
The computational complexity of each part in \modelname~is analyzed as follows.

In Section~\ref{subsec:graph_encoding}, the computation complexity of the \textbf{dynamic graph encoding} operations is:
\begin{equation}\label{Eq:eadgnn}
    \mathcal{O} \left ( |\mathcal{E}| \sum_{l=0}^{L} d^{(l)} + \mathcal{V} \left ( \sum_{l=1}^{L} d^{(l-1)}d^{(l)} + \left(d^{(L)}\right)^2 \right ) \right ),
\end{equation}
where $d^{(l)}$ represents the node representation dimensionality in the $l$-th layer. 
Given that $L$ and $d^{(l)}$ are small and fixed constants, we can simplify Eq.~\eqref{Eq:eadgnn} to $\mathcal{O}(|\mathcal{E}|d+|\mathcal{V}|d^2)$, where $d$ denotes the common dimensionality across all layers.

In Section~\ref{subsec:ESVAE}, we introduce \textbf{ESVAE}, which comprises an encoder, an LSTM model for capturing dynamic latent variables, and a decoder. 
The encoder and decoder share a computational complexity of $\mathcal{O}(|\mathbf{H}|Ld)$, where $|\mathbf{H}|$ denotes the count of observed environment instances, and $L$ signifies the number of layers in both the encoder and decoder. 
The LSTM model's complexity is $\mathcal{O}(Td^2)$, with $T$ being the total number of graph snapshots. 
For simplicity, we omit $L$ from the notation since it is typically a small value. 
Therefore, the overall complexity of ESVAE is $\mathcal{O}(|\mathbf{H}|d+Td^2)$.

In Section~\ref{subsec:IPR}, the \textbf{invariant pattern recognition} mechanism is applied to all nodes using $\mathbb{I}(\cdot)$, which operates with a computational complexity of $\mathcal{O}(d\log|\mathcal{V}|)$, where $d$ represents the dimensionality of the node representations.

In Section~\ref{subsec:intervention}, we implement \textbf{causal interventions} through sampling and replacement procedures. 
Here, $|\mathcal{E}|_p$ denotes the edge numbers to be predicted, and $|\mathcal{S}|$ represents the size of the intervention set, typically a small fixed constant. 
The computational complexity of the causal intervention mechanism, when considering both sampling and replacement, is $\mathcal{O}(|\mathcal{S}|d) + \mathcal{O}(|\mathcal{E}_p| |\mathcal{S}|d)$ during training. 
Notably, this mechanism incurs no additional computational complexity during the testing stage.

Then \modelname's overall computation complexity is:
\begin{equation}
\begin{aligned}
    &\mathcal{O}(|\mathcal{E}|d+|\mathcal{V}|d^2) + \mathcal{O}(|\mathbf{H}|d+Td^2) + \mathcal{O}(d\log|\mathcal{V}|) \\
    &+ \mathcal{O}( |\mathcal{S}|d) + \mathcal{O}(|\mathcal{E}_p| |\mathcal{S}|d).    
\end{aligned}
\end{equation}

In conclusion, the computational complexity of \modelname~is linear with the number of nodes and edges, aligning with the complexity of EAGLE~\cite{yuan2023environment} and DIDA~\cite{zhang2022dynamic}, as well as other dynamic graph neural networks.

\begin{table*}[!ht]
  \caption{AUC score  (\% ± standard deviation) and AUC decrease  ($\Delta$) of future link prediction task on real-world datasets with OOD shifts of link attributes. (Best results: \textbf{bold}. Runner-ups: \underline{underlined}.)}
  \label{tab:res_link}
  \resizebox{\linewidth}{!}{
    \centering
      \begin{tabular}{c|ccr|ccr|ccr}
      \toprule
      \textbf{Dataset} & \multicolumn{3}{c|}{\textbf{COLLAB}} & \multicolumn{3}{c|}{\textbf{Yelp}} & \multicolumn{3}{c}{\textbf{ACT}}  \\
  \cmidrule{1-10}    \textbf{Model} & \textit{w/o OOD}& \textit{w/ OOD}  & \multicolumn{1}{c|}{$\Delta $} &  \textit{w/o OOD} & \textit{w/ OOD} & \multicolumn{1}{c|}{$\Delta $} &  \textit{w/o OOD} & \textit{w/ OOD} & \multicolumn{1}{c}{$\Delta $}    \\
      \midrule
      GCRN~\cite{seo2018structured}  & 82.78±0.54 & 69.72±0.45 & 13.06 (15.78\%) $\downarrow$ & 68.59±1.05 & 54.68±7.59 & 13.91 (20.28\%) $\downarrow$ & 76.28±0.51  & 64.35±1.24 & 11.93 (15.64\%) $\downarrow$\\
      EvolveGCN~\cite{pareja2020evolvegcn} & 86.62±0.95 & 76.15±0.91 & 10.47 (12.09\%) $\downarrow$ & 78.21±0.03 & 53.82±2.06 & 24.39 (31.19\%) $\downarrow$ & 74.55±0.33 & 63.17±1.05 & 11.38 (15.26\%) $\downarrow$ \\
      DySAT~\cite{sankar2020dysat} & 88.77±0.23 & 76.59±0.20 & 12.18 (13.72\%) $\downarrow$ & 78.87±0.57 & 66.09±1.42 & 12.78 (16.20\%) $\downarrow$ & 78.52±0.40 & 66.55±1.21 & 11.97 (15.24\%) $\downarrow$\\
      \midrule
      IRM~\cite{arjovsky2019invariant}   & 87.96±0.90 & 75.42±0.87 & 12.54 (14.26\%) $\downarrow$ & 66.49±10.78 & 56.02±16.08 & 10.47 (15.75\%) $\downarrow$ & 80.02±0.57 & 69.19±1.35 & 10.83 (13.53\%) $\downarrow$\\
      V-REx~\cite{krueger2021out} & 88.31±0.32 & 76.24±0.77 & 12.07 (13.67\%) $\downarrow$&79.04±0.16 & 66.41±1.87 & 12.63 (15.98\%) $\downarrow$ & 83.11±0.29 & 70.15±1.09 & 12.69 (15.27\%) $\downarrow$\\
      GroupDRO~\cite{sagawa2019distributionally} & 88.76±0.12 & 76.33±0.29 & 12.43 (14.00\%) $\downarrow$ &\underline{79.38±0.42} & 66.97±0.61 & 12.41 (15.63\%) $\downarrow$ & 85.19±0.53 & 74.35±1.62 & 10.84 (12.72\%) $\downarrow$\\
      \midrule
      SR-GNN~\cite{zhu2021shift}  & 88.39±0.03 & 76.99±0.49 &  11.40 (12.90\%) $\downarrow$  & 72.04±0.59 & 68.15±1.27 & 3.89 (5.40\%) $\downarrow$ & 86.75±0.07  & 76.86±0.19 & 9.89 (11.40\%) $\downarrow$ \\
      EERM~\cite{wu2022handling}  & 88.02±0.24 & 77.84±0.21 &  10.18 (1.57\%) $\downarrow$ & 61.90±0.84 & 69.82±3.69 & \textbf{7.92 (12.97\%) $\uparrow$} & 86.71±0.18  & 74.07±0.74 & 12.64 (14.58\%) $\downarrow$\\
      \midrule
      DIDA~\cite{zhang2022dynamic}  & 91.97±0.05 & 81.87±0.40 & 10.10 (10.98\%) $\downarrow$ & 78.22±0.40 & 75.92±0.90 & 2.30(2.94\%) $\downarrow$ & 89.84±0.82 & 78.64±0.97 & 11.20 (12.47\%) $\downarrow$\\
      SILD~\cite{zhang2023spectral}  & 89.68±0.97 & 81.69±1.30 &  
      \textbf{7.99 (8.91\%) $\downarrow$}  & 73.05±0.82 & 77.09±1.58 & \underline{4.04 (5.53\%) $\uparrow$} & 87.80±1.24  & 80.91±1.19 & \textbf{6.89 (7.84\%)} $\downarrow$ \\
      EAGLE~\cite{yuan2023environment}  & \underline{92.45±0.21} & \underline{84.41±0.87} &  \underline{8.04 (8.70\%) $\downarrow$}  & 78.97±0.31 & \underline{77.26±0.74} & 1.71 (2.17\%) $\downarrow$&  \underline{92.37±0.53} & \underline{82.70±0.72} & \underline{9.67 (10.47\%) $\downarrow$}\\
      \midrule
      \textbf{\modelname} & \textbf{93.44±0.05} & \textbf{84.46±0.03} & 8.98 (9.61\%) $\downarrow$ & \textbf{79.71±1.26}  & \textbf{78.74±1.50} & 0.97 (1.21\%) $\downarrow$ & \textbf{92.68±0.01} & \textbf{82.94±0.03} & 9.74 (10.51\%) $\downarrow$\\
      \bottomrule
      \end{tabular}%
  }
  \end{table*}%

\section{Experiments}
\label{sec:exp}

We conduct experiments to demonstrate the efficacy of \modelname~for handling the distribution shifts in dynamic graphs, aiming to answer the following questions: 
\begin{itemize}[leftmargin=1.5em]
    \item \textbf{Q1.} 
    How does \modelname~perform in the future link prediction task with distribution shifts? 
    (Section \ref{subsec:LP}) 
    \item \textbf{Q2.} 
    How does \modelname~perform in the node classification task with distribution shifts? 
    (Section \ref{subsec:NC})     
    \item \textbf{Q3.} 
    How do the proposed environment sequential variational auto-encoder and the causal Intervention mechanism affect \modelname's performance? 
    (Section \ref{subsec:module_effectiveness}) 
    \item \textbf{Q4.} 
    How does the environment-aware invariant pattern recognition affect \modelname's performance? 
    (Section~\ref{subsec:invariant_recognition})  
\end{itemize}
\subsection{Experimental Setups}
\subsubsection{Datasets}
\label{subsubsec:datasets}
\textbf{Datasets for the future link prediction task. }
We assess the performance of \modelname\footnote{The code of \modelname~is available at \url{https://github.com/RingBDStack/EvoGOOD}.} with three real-world datasets.
\textbf{COLLAB}~\cite{tang2012cross} is an academic collaboration dataset within 16 years. 
\textbf{Yelp}~\cite{sankar2020dysat} is a customer review dataset on business within 24 months. 
\textbf{ACT}~\cite{kumar2019predicting} is a students' action dataset of a MOOC platform within 30 days. 
These aforementioned datasets span different periods and vary in time granularity (year, month, and day), covering a broad spectrum of real-world scenarios.

\textbf{Datasets for the node classification task. }
We generate synthetic datasets by the stochastic block model\cite{holland1983stochastic} as described in~\cite{zhang2023spectral}. 
The probability of link occurrence in each snapshot depends on two components: the correlation factor and the variant factor. 
The correlation factor's association with class labels is consistently set to 1, whereas the variant factor's relationship with labels varies and is modulated by a shift-level parameter. 
A generalized model should be able to identify the invariant frequency factors that maintain a consistent relationship with labels despite distribution shifts.

\subsubsection{Baselines}
Baselines include representative graph learning methods, OOD generalization methods, and graph OOD methods. 
\begin{itemize}[leftmargin=1.5em]
    \item \textbf{Dynamic Graph Learning Methods}: 
    GCRN~\cite{seo2018structured} adopts GCN and GRU to capture temporal relations for node embeddings. 
    EvolveGCN~\cite{pareja2020evolvegcn} uses an LSTM to learn the evolution of the GCN's parameters. 
    DySAT~\cite{sankar2020dysat} uses self-attentions to learn spatial-temporal patterns. 
    \item \textbf{OOD Generalization Methods}: IRM~\cite{arjovsky2019invariant} aims to minimize invariant risk with an invariant predictor.
    V-REx~\cite{krueger2021out} proposes to reweight the risk to further extend IRM. GroupDRO~\cite{sagawa2019distributionally} learn to reduce the risk gap across training distributions. 
    \item \textbf{Static graph OOD Methods}: 
    SR-GNN~\cite{zhu2021shift} uses GNN to recognize the distribution shifts between the training set and the true inference distribution. 
    EERM~\cite{wu2022handling} maximizes the risk variance from multiple environments by leveraging invariance principles with graph neural networks adversarially. 
    \item \textbf{Dynamic graph OOD Methods}: 
    DIDA~\cite{zhang2022dynamic} is the first work that exploits invariant patterns on dynamic graphs by a disentangled mechanism.
    SILD~\cite{zhang2023spectral} addresses the dynamic graph distribution shifts by identifying and leveraging invariant/variant patterns in the spectral domain.
    EAGLE~\cite{yuan2023environment} models the coupled environments and performs fine-grained causal interventions with instantiated environment samples.

\end{itemize}

\subsection{Evaluation of Future Link Prediction Task (Q1)}
\label{subsec:LP}

\begin{table*}[!ht]
  \caption{AUC score  (\% ± standard deviation) and AUC decrease  ($\Delta$) of future link prediction task on real-world datasets with OOD shifts of node features. 
  (Best results: \textbf{bold}. Runner-ups: \underline{underlined}.)
  }
  \label{tab:res_node}
  \resizebox{\linewidth}{!}{
    \centering
      \begin{tabular}{c|ccr|ccr|ccr}
      \toprule
      \textbf{Dataset} & \multicolumn{3}{c|}{\textbf{COLLAB  ($\bar{p}=$~0.4)}} & \multicolumn{3}{c|}{\textbf{COLLAB  ($\bar{p}=$~0.6)}} & \multicolumn{3}{c}{\textbf{COLLAB  ($\bar{p}=$~0.8)}}\\
  \cmidrule{1-10}    \textbf{Model} & \multicolumn{1}{c}{Train} & \multicolumn{1}{c}{Test} &\multicolumn{1}{c|}{$\Delta $} & Train & Test &\multicolumn{1}{c|}{$\Delta$}  & Train & Test & \multicolumn{1}{c}{$\Delta $}  \\
      \midrule
  GCRN~\cite{seo2018structured}  & {\hspace{0.0em}69.60±1.14\hspace{0.0em}} & {\hspace{0.0em}72.57±0.72\hspace{0.0em}} & \underline{ 
  2.97 (4.27\%) $\uparrow$} &{\hspace{0.0em}74.71±0.17\hspace{0.0em}} & {\hspace{0.0em}72.29±0.47\hspace{0.0em}} &  
 2.42 (3.24\%) $\downarrow$ &{\hspace{0.0em}75.69±0.07\hspace{0.0em}} & {\hspace{0.0em}67.26±0.22\hspace{0.0em}} & 8.43 (11.14\%) $\downarrow$\\
      EvolveGCN~\cite{pareja2020evolvegcn} & 78.82±1.40 & 69.00±0.53 & 9.82 (12.46\%) $\downarrow$ & 79.47±1.68 & 62.70±1.14 & 16.77 (21.10\%) $\downarrow$& 81.07±4.10 & 60.13±0.89& 20.94 (25.83\%) $\downarrow$ \\
      DySAT~\cite{sankar2020dysat} & 84.71±0.80 & 70.24±1.26 & 14.47 (17.08\%) $\downarrow$ & 89.77±0.32 & 64.01±0.19 & 25.76 (28.70\%) $\downarrow$ & 94.02±1.29 & 62.19±0.39 & 31.83 (33.85\%) $\downarrow$\\
      \midrule
      IRM~\cite{arjovsky2019invariant}   & 85.20±0.07 & 69.40±0.09&15.8 (18.54\%) $\downarrow$ & 89.48±0.22 & 63.97±0.37&25.51 (28.51\%) $\downarrow$ &  \textbf{95.02±0.09} & 62.66±0.33&32.36 (34.06\%) $\downarrow$\\
      V-REx~\cite{krueger2021out} & 84.77±0.84 & 70.44±1.08&14.33 (16.90\%) $\downarrow$ & 89.81±0.21 & 63.99±0.21&25.82 (28.75\%) $\downarrow$ & 94.06±1.30 & 62.21±0.40& 31.85 (33.86\%) $\downarrow$\\
      GroupDRO~\cite{sagawa2019distributionally} & 84.78±0.85 & 70.30±1.23& 14.48 (17.08\%) $\downarrow$ & 89.90±0.11 & 64.05±0.21 &25.85 (28.75\%) $\downarrow$ & 94.08±1.33& 62.13±0.35& 31.95 (33.96\%) $\downarrow$\\
      \midrule
      SR-GNN~\cite{zhu2021shift}  & 75.96±1.19 & 70.95±1.10 &  5.01 (6.59\%) $\downarrow$  & 79.66±0.94 & 70.95±1.57 & 8.71 (10.93\%) $\downarrow$ & 84.73±1.10  & 71.00±1.23 & 13.73 (16.20\%) $\downarrow$ \\
      EERM~\cite{wu2022handling}  & 81.73±0.30 & 62.32±0.28 & 19.41 (23.74\%) $\downarrow$   & 87.42±0.22 & 62.47±0.24 & 24.95 (28.54\%) $\downarrow$ & 91.98±0.42  & 62.55±0.21 & 29.43 (32.00\%) $\downarrow$\\
      \midrule
      DIDA~\cite{zhang2022dynamic}  & 87.92±0.92 & 85.20±0.84 &   2.72 (3.09\%) $\downarrow$ & 91.22±0.59 & 82.89±0.23&8.33 (9.13\%) $\downarrow$ &  92.72±2.16 & 72.59±3.31& 20.13 (21.71\%) $\downarrow$ \\
      SILD~\cite{zhang2023spectral}  & 78.47±1.32 & 85.30±1.02 & \textbf{6.83 (8.70\%) $\uparrow$}   & 78.98±2.40 & 85.06±1.48 & \textbf{6.08 (7.70\%) $\uparrow$} &  79.49±1.66 & \underline{85.88±0.93} & \textbf{6.39 (8.04\%) $\uparrow$} \\
      EAGLE~\cite{yuan2023environment}  & \underline{92.97±0.88} & \underline{88.32±0.61} & 4.65 (5.00\%) $\downarrow$ & \textbf{94.52±0.42} & \underline{87.29±0.71} & 7.23 (7.65\%) $\downarrow$&  \underline{94.11±1.03} & 82.30±0.75 & 11.81 (12.55\%) $\downarrow$\\
      \midrule
      \textbf{\modelname} & \textbf{93.13±0.03} & \textbf{93.12±0.03}& 0.01 (0.01\%) $\downarrow$ & \underline{93.02±0.04} & \textbf{93.01±0.05}&\underline{0.01 (0.01\%) $\downarrow$} & 93.02±0.04 & \textbf{93.01±0.04}& \underline{0.01 (0.01\%) $\downarrow$}\\
      \bottomrule
      \end{tabular}%
  }
  \end{table*}%

\begin{figure*}[!htbp]
\centering
\subfigure[Results on COLLAB.]{
\begin{minipage}[t]{ 0.31\linewidth}
\centering
\includegraphics[width=\linewidth]{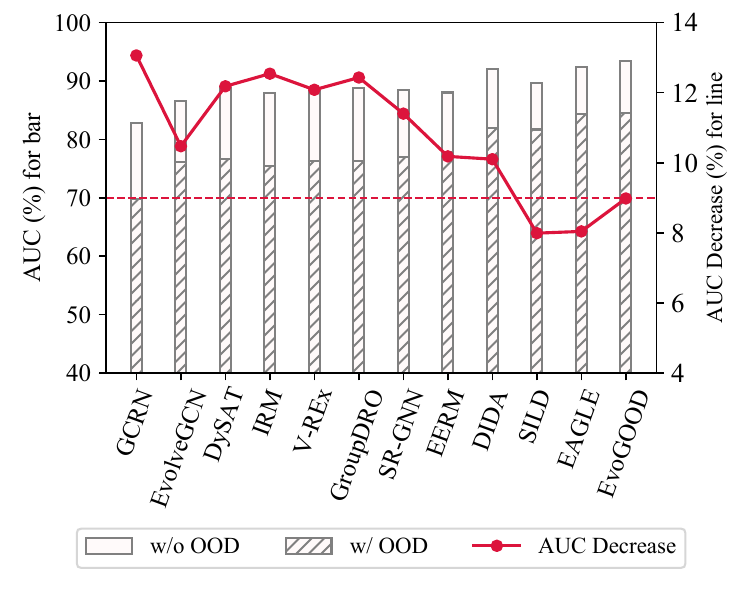}
\end{minipage}
}
\subfigure[Results on Yelp.]{
\begin{minipage}[t]{ 0.31\linewidth}
\centering
\includegraphics[width=\linewidth]{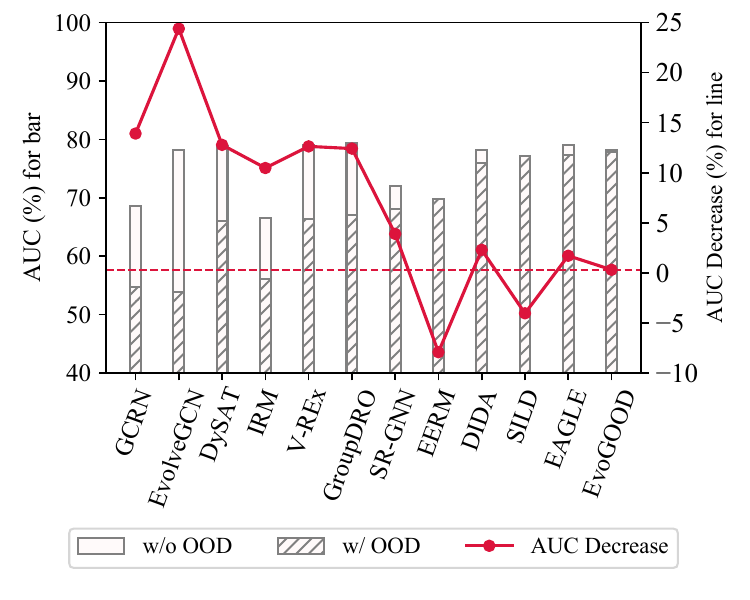}
\end{minipage}
}
\subfigure[Results on ACT.]{
\begin{minipage}[t]{ 0.31\linewidth}
\centering
\includegraphics[width=\linewidth]{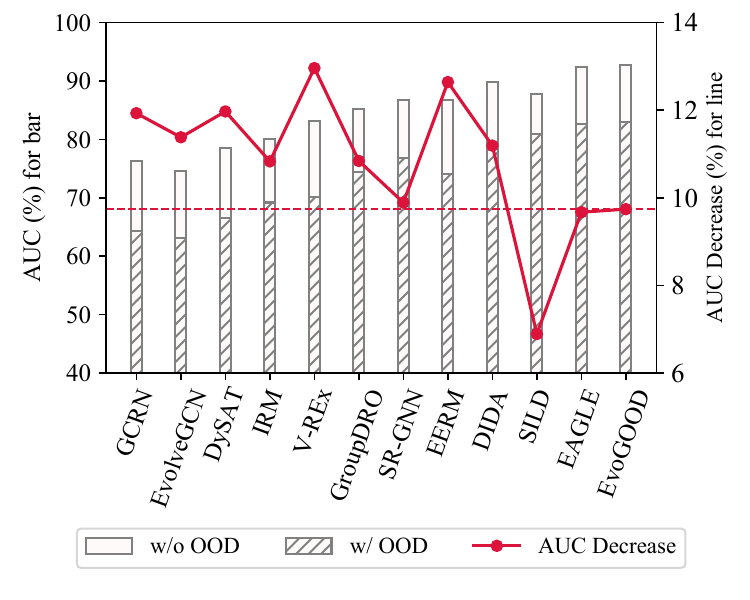}
\end{minipage}
}
\quad
\subfigure[Results on COLLAB ($\bar{p}=0.2$).]{
\begin{minipage}[t]{ 0.31\linewidth}
\centering
\includegraphics[width=\linewidth]{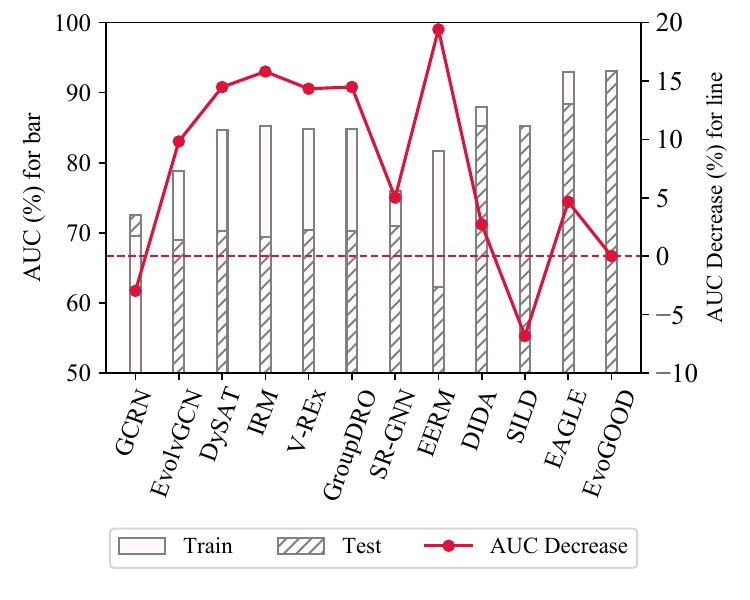}
\end{minipage}
}
\subfigure[Results on COLLAB ($\bar{p}=0.4$).]{
\begin{minipage}[t]{ 0.31\linewidth}
\centering
\includegraphics[width=\linewidth]{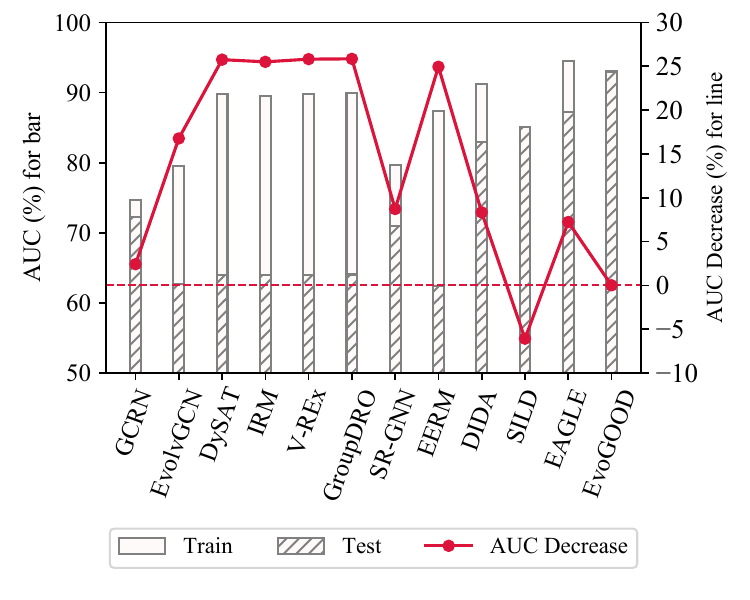}
\end{minipage}
}
\subfigure[Results on COLLAB ($\bar{p}=0.8$).]{
\begin{minipage}[t]{ 0.31\linewidth}
\centering
\includegraphics[width=\linewidth]{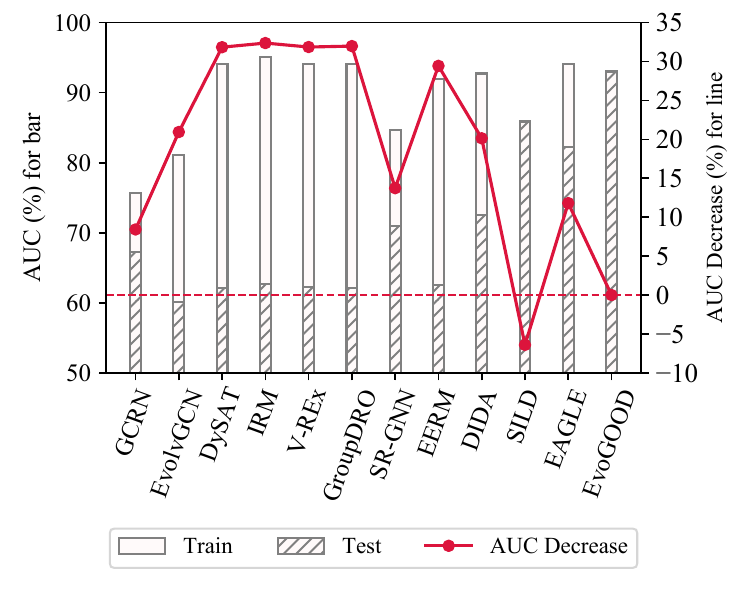}
\end{minipage}
}
\centering
\caption{Additional analysis of the performance. }
\vspace{-0.6em}
\label{fig:additional_analysis}
\end{figure*}

We evaluate the performance of \modelname~under distribution shifts of link attribute and node feature. 
\subsubsection{Distribution Shifts on Link Attributes}
\label{subsec:link_attribute}
\textbf{Settings. }
Following~\cite{zhang2022dynamic}, we selectively exclude specific attributes from the multi-attribute relations to act as shifted factors under OOD environments in the validation and training sets over time. 
The model has no access to any information regarding the excluded links until the testing phase, making the task more realistic and demanding in terms of generalization capabilities.
\change{For node-wise interventions, we sample environment instances from both the observed library $S_{ob}$ and the generated library $S_{ge}$. At each training step, an intervention ratio $s\%$ of the nodes is randomly selected to perform causal replacement according to Eq.~\ref{eq:intervention}. We set $s$ as a constant to balance training efficiency and effectiveness, ensuring that the intervention cost grows linearly with the number of nodes and predicted edges.}

\textbf{Results.} 
Table \ref{tab:res_link} presents the results, with \textit{w/o OOD} indicating testing without distribution shifts and \textit{w/ OOD} referring to testing under distribution shifts, $\Delta$(\%) denotes the AUC decrease and the corresponding percentage (the lower, the better). 
With the investigation of evolving environments, our \modelname~shows the best performance in all \textit{w/ OOD} datasets and two \textit{w/o OOD} datasets. 
Particularly on the most challenging dataset COLLAB, which has the longest period and very different link attribute distribution, \modelname~shows the best generalization ability. 
On all datasets, the performance of all baselines decreases dramatically with distribution shifts. 
Both static and dynamic baselines without OOD generalization (GAE, VGAE, GCRN, EvolveGCN, and DySAT) fail in distribution shifts, primarily due to their reliance on variant patterns with spurious correlations. 
Traditional OOD generalization baselines, such as IRM, V-REx, and GroupDRO, offer only modest improvements since they require environment labels, which are not provided in dynamic graph settings.
As the prior works for dynamic graph OOD generalization, DIDA~\cite{zhang2022dynamic} and EAGLE~\cite{yuan2023environment} show further progress by learning invariant patterns under distribution shifts. 
However, without considering the relation between invariant patterns and environments, its generalization ability is not as well as our \modelname~because of the label shift phenomenon~\cite{yu2023mind}.

\subsubsection{Distribution Shifts of Node Features}
\textbf{Settings. }
Here we focus on the most challenging dataset COLLAB. 
Following~\cite{wu2022handling,zhang2022dynamic}, to induce node feature shifts, we uniformly sample $\Tilde{\mathbf{A}}^{t+1}$ from the $|\mathcal{E}^{t+1}|$ links at the next time-step in COLLAB with the sampling probability $p(t)=\bar{p}+\sigma \cos (t)$. 
Then we factorize links $\Tilde{\mathbf{A}}^{t+1}$ into synthetic shifted features $\Tilde{\mathbf{X}^{t}} \in \mathbb{R}^{N\times d}$ by optimizing the reconstruction cross-entropy loss $\ell(\Tilde{\mathbf{X}^{t}} \Tilde{\mathbf{X}}^{t\top}, \Tilde{\mathbf{A}}^{t+1})$. 
The synthetic node features are concatenated with the original ones to form the input while maintaining the original graph structure. 
$\Tilde{\mathbf{X}}^{t}$ with a larger $p(t)$ exhibits stronger spurious correlations with future environments. 
We set $\bar{p}$ to 0.4, 0.6, or 0.8 for the training set and 0.1 for the test set, indicating that the training graphs contain more spurious correlations.
\change{We adopt the same intervention procedure by sampling from observed and generated instances as in Section~\ref{subsec:link_attribute}.}

\textbf{Results.} 
Table \ref{tab:res_node} shows the results, where $\Delta$(\%) denotes the AUC decrease and the corresponding percentage. 
The results of static GNNs are omitted as they cannot support dynamic node features. 
It is evident that the testing performance of all baselines decreases notably. 
In contrast, our \modelname~manages distribution shifts in node features on dynamic graphs more effectively and exhibits significantly smaller performance gaps. 
Notably, \modelname~outperforms the best baseline, EAGLE, by roughly 5\%, 6\%, and 11\% in AUC on the testing set across different shifting levels. 
Even on the most challenging dataset, COLLAB, with a high level of spurious correlations $(\bar{p}=0.8)$, our \modelname~maintains satisfactory performance, indicating its superior ability to mitigate spatio-temporal spurious correlations in highly challenging OOD settings.

We visualize Table~\ref{tab:res_link} and Table~\ref{tab:res_node} to provide additional analysis. 
Fig.~\ref{fig:additional_analysis} visualizes the task performance (AUC) under \textit{w/o OOD} and \textit{w/ OOD} settings for link attribute distribution shifts and the task performance under training set with more spurious correlations and test set with less spurious correlations.  
Additionally, Fig.~\ref{fig:additional_analysis} illustrates the decreasing value of AUC for baselines, with the horizontal dashed line indicating the decreasing value of AUC for \modelname. 
A smaller decrease suggests that the methods have a more robust generalization ability under OOD shifts. 
On the vast majority of datasets, \modelname~can improve the task performance while minimizing AUC decrease. 
In most cases, the AUC decreasing value of our \modelname~exceeds the baseline except for SILD. 
Although SILD has satisfied generalization ability, its task performance is inherently poorer. 
In summary, \modelname~shows consistent advantage in terms of both performance and generalization ability. 

\begin{figure*}[!htbp]
\centering
\subfigure[Ablation study results for \modelname.]{
\begin{minipage}[t]{ 0.31\linewidth}
\centering
\includegraphics[width=\linewidth]{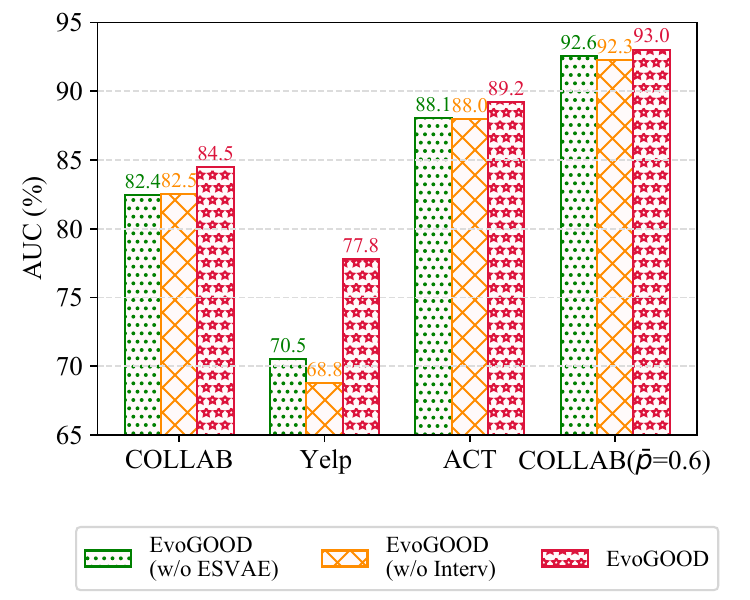}
\label{fig:abla}
\end{minipage}
}
\subfigure[Parameter sensitivity of ESVAE loss.]{
\begin{minipage}[t]{ 0.31\linewidth}
\centering
\includegraphics[width=\linewidth]{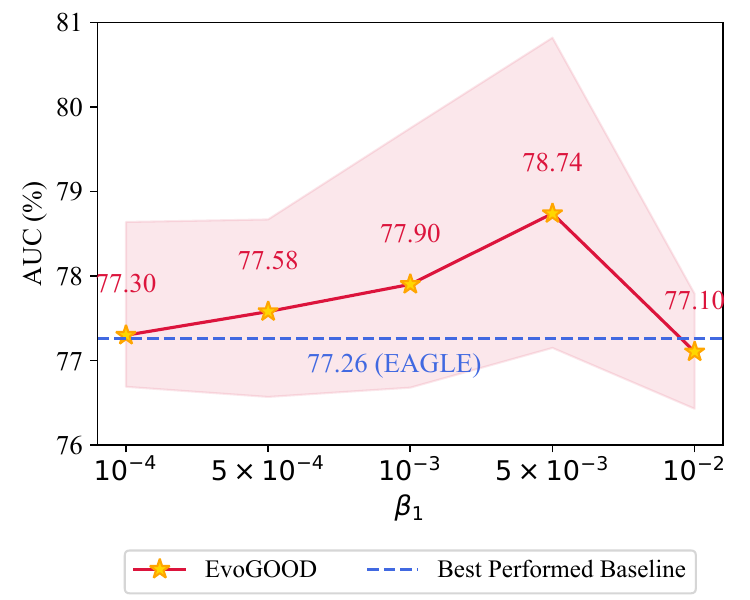}
\label{fig:vae_loss}
\end{minipage}
}
\subfigure[Parameter sensitivity of intervention loss.]{
\begin{minipage}[t]{ 0.31\linewidth}
\centering
\includegraphics[width=\linewidth]{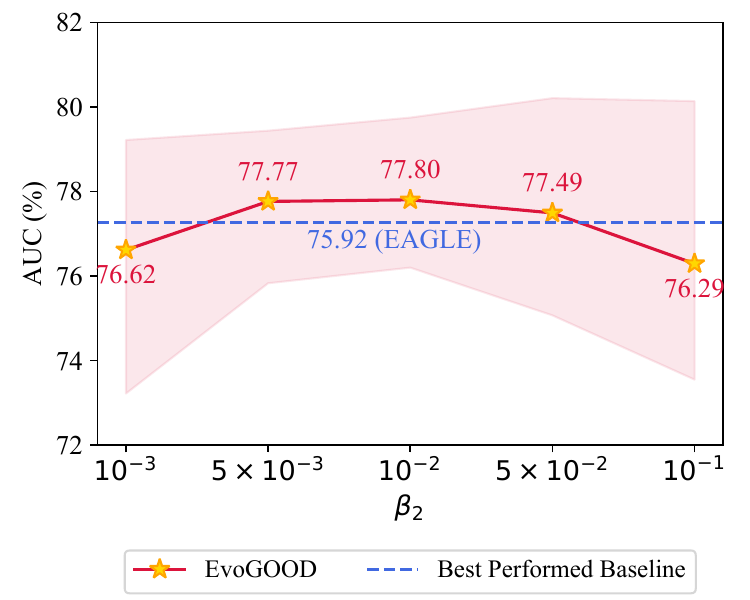}
\label{fig:risk_loss}
\end{minipage}
}
\centering
\caption{Performance analysis of \modelname. }
\vspace{-0.6em}
\label{fig:exp}
\end{figure*}

\subsection{Evaluation of Node Classification Task (Q2)}
\label{subsec:NC}
\textbf{Settings. }
We conduct additional node classification experiments on synthetic OOD datasets as mentioned in Section~\ref{subsubsec:datasets}. 
We assign shift-level parameters of 0.4, 0.6, and 0.8 to the training and validation sets, and a parameter value of 0 to the test set.
Table~\ref{tab:node_cls} shows the results (Accuracy \%). 

\begin{table}[!htp]
\caption{Accuracy (\% ± standard deviation) of node classification task on synthetic datasets with different OOD shift levels. (Best results: \textbf{bold}. Runner-ups: \underline{underlined}.)}
\label{tab:node_cls}
\resizebox{\linewidth}{!}{
\centering
\begin{tabular}{cccc}
\toprule
\textbf{Dataset} & \textbf{Node-Syn-0.4}        & \textbf{Node-Syn-0.6}        & \textbf{Node-Syn-0.8}        \\ \midrule
GCRN~\cite{seo2018structured}          & 27.19±2.18 & 25.95±0.80 & 29.26±0.69 \\
EGCN~\cite{pareja2020evolvegcn}          & 24.01±2.29 & 22.75±0.96 & 24.98±1.32 \\
DySAT~\cite{sankar2020dysat}         & 40.95±2.89 & 37.94±1.01 & 30.90±1.97 \\ \midrule
IRM~\cite{arjovsky2019invariant}           & 33.23±4.70 & 30.29±1.71 & 29.43±1.38 \\
VREx~\cite{krueger2021out}         & 41.78±1.30 & 38.11±2.81 & 29.56±0.44 \\
GroupDRO~\cite{sagawa2019distributionally}      & 41.35±2.19 & 35.74±3.93 & 31.01±1.24 \\ \midrule
SR-GNN~\cite{zhu2021shift}        &  41.33±2.47          &     40.60±3.45       &    \underline{39.42±0.71}        \\
EERM~\cite{wu2022handling}          &    37.18±2.61        &     34.81±3.79       & 34.02±1.71           \\ \midrule
DIDA~\cite{zhang2022dynamic}          & 43.33±7.74 & 39.48±7.93 & 28.14±3.07 \\
SILD~\cite{zhang2023spectral}          & \underline{43.62±2.74} & 39.78±3.56 & 38.64±2.76 \\
EAGLE~\cite{yuan2023environment}         & 41.37±0.68 & \underline{41.11±1.86} & 37.16±0.90 \\ \midrule
\textbf{EvoGOOD (ours) }      & \textbf{44.30±4.61} & \textbf{42.89±1.42} & \textbf{39.81±2.14} \\ \bottomrule
\end{tabular}
}
\end{table}

\textbf{Results. }
We can observe that \modelname~shows the best performance on the node classification task under all OOD levels compared with all baselines. 
As the correlation between variant patterns and labels intensifies, the model tends to rely more on these variant patterns during training, which can hinder its generalization capability. 
As the level of distribution shift increases, we observe a performance decline across nearly all methods; however, the performance decrease for \modelname~is notably less than that of the baselines. This indicates that \modelname~can effectively identify invariant patterns within different environments and mitigate the adverse effects of variant patterns.

\subsection{Effectiveness and Efficiency of \modelname~(Q3)}
\label{subsec:module_effectiveness}
\subsubsection{Ablation Study}
\label{subsubsec:abla}
In this subsection, we conduct ablation studies to analyze the effectiveness of the proposed ESVAE and the causal intervention mechanisms: 
\begin{itemize}[leftmargin=*]
    \item \textbf{\modelname~(\textit{w/o ESVAE})} removes the ESVAE proposed in Section~\ref{subsec:ESVAE} and uses a shared variational auto-encoder across all graph snapshots to learn the environment distribution. 
    \item \textbf{\modelname~(\textit{w/o Interv})} removes the causal intervention mechanism in Section~\ref{subsec:intervention} and is optimized by Eq.~\eqref{eq:final_loss} without the $\mathcal{L}_{\mathrm{risk}}$ term. 
\end{itemize}
Fig.~\ref{fig:abla} shows the results. 
Overall, \modelname~shows consistent advantages over the other two variants on all datasets, showing the advantage of learning the environment evolution and performing corresponding causal interventions. 
Especially on Yelp, the proposed intervention mechanism with inferred environment distribution dramatically improves the generalization ability of the method. 

\change{To further validate our contributions, we investigate the improvements achieved when incorporating the proposed ESVAE (Environment Sequential Variational Autoencoder in Section~\ref{subsec:ESVAE}) and IPR (Invariant Pattern Recognition in Section~\ref{subsec:IPR}) into the original EAGLE~\cite{yuan2023environment} presented in the previous conference version. The results are shown in Table~\ref{tab:improvement_over_conference_1} and Table~\ref{tab:improvement_over_conference_2}.}
\begin{table}[htbp]
  \centering
  \caption{\change{AUC score improvements (\% ± standard deviation) obtained by incorporating the proposed ESVAE and IPR modules into the original EAGLE~\cite{yuan2023environment} presented in the conference version, on future link prediction tasks over real-world datasets with OOD shifts in \textbf{link attributes}.}}
  \label{tab:improvement_over_conference_1}
  {\color{\changecolor}
  \begin{tabularx}{\linewidth}{l *{3}{>{\centering\arraybackslash}X}}
    \toprule
    \textbf{Dataset} & \textbf{COLLAB} & \textbf{YELP} & \textbf{ACT} \\
    \midrule
    EAGLE & 84.41±0.87 & 77.26±0.74 & 82.70±0.72 \\
    EAGLE(+ESVAE) & 84.45±0.28 & 77.78±0.82 & 83.20±0.45 \\
    EAGLE(+IPR) & 84.09±0.25 & 77.64±0.30 & 83.26±0.33 \\
    \bottomrule
  \end{tabularx}
  }
\end{table}

\begin{table}[htbp]
  \centering
  \caption{\change{AUC score improvements (\% ± standard deviation) obtained by incorporating the proposed ESVAE and IPR modules into the original EAGLE~\cite{yuan2023environment} presented in the conference version, on future link prediction tasks over real-world datasets with OOD shifts in \textbf{node features}.}}
  \label{tab:improvement_over_conference_2}
  {\color{\changecolor}
  \begin{tabularx}{\linewidth}{l *{3}{>{\centering\arraybackslash}X}}
    \toprule
    \textbf{Dataset} & \makecell{\textbf{COLLAB}\\($\bar{p}=$~0.4)}  & \makecell{\textbf{COLLAB}\\($\bar{p}=$~0.6)}  & \makecell{\textbf{COLLAB}\\($\bar{p}=$~0.8)}  \\
    \midrule
    EAGLE & 88.32±0.61 & 87.29±0.71 & 82.30±0.75 \\
    EAGLE(+ESVAE) & 89.16±0.36 & 87.68±0.14 & 83.03±0.18 \\
    EAGLE(+IPR) & 89.44±0.30 & 87.77±0.31 & 83.35±0.42 \\
    \bottomrule
  \end{tabularx}
  }
\end{table}

\change{\textbf{Results.} We observe that incorporating the proposed modules into the original EAGLE leads to consistent improvements on future link prediction tasks under OOD shift. As shown in Table 5, adding either ESVAE or IPR yields higher AUC scores than the baseline EAGLE across almost all datasets, except for the COLLAB dataset, when the distribution shift occurs in link attributes. Similarly, Table 6 demonstrates that both modules bring noticeable gains under OOD shifts in node features. In particular, ESVAE consistently enhances the robustness of EAGLE, while IPR further contributes complementary improvements. These results indicate that both modules effectively strengthen the generalization ability of EAGLE by mitigating the adverse impact of distribution shifts in different feature spaces.}

\begin{figure*}[!htbp]
\centering
\subfigure[Efficiency comparision.]{
\begin{minipage}[t]{ 0.31\linewidth}
\centering
\includegraphics[width=\linewidth]{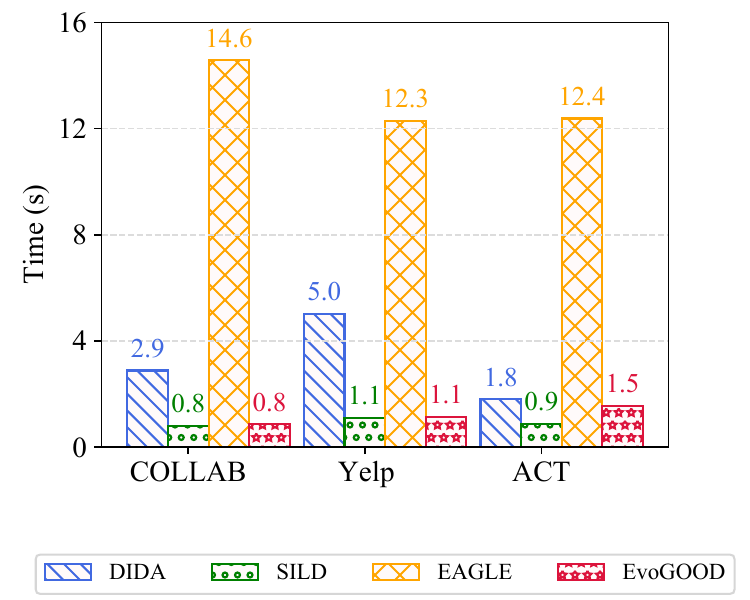}
\label{fig:efficiency}
\end{minipage}
}
\subfigure[AUC with different invariant levels.]{
\begin{minipage}[t]{ 0.31\linewidth}
\centering
\includegraphics[width=\linewidth]{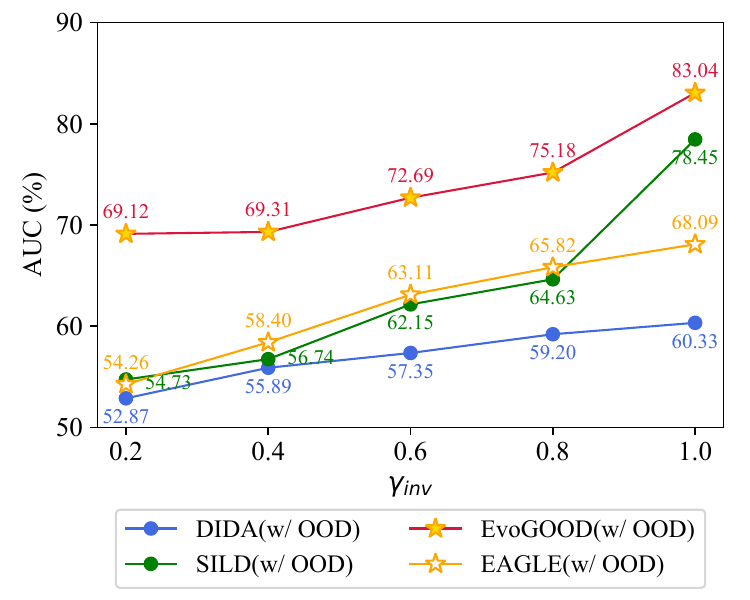}
\label{fig:invariant_level}
\end{minipage}
}
\subfigure[AUC with different evolving levels.]{
\begin{minipage}[t]{ 0.31\linewidth}
\centering
\includegraphics[width=\linewidth]{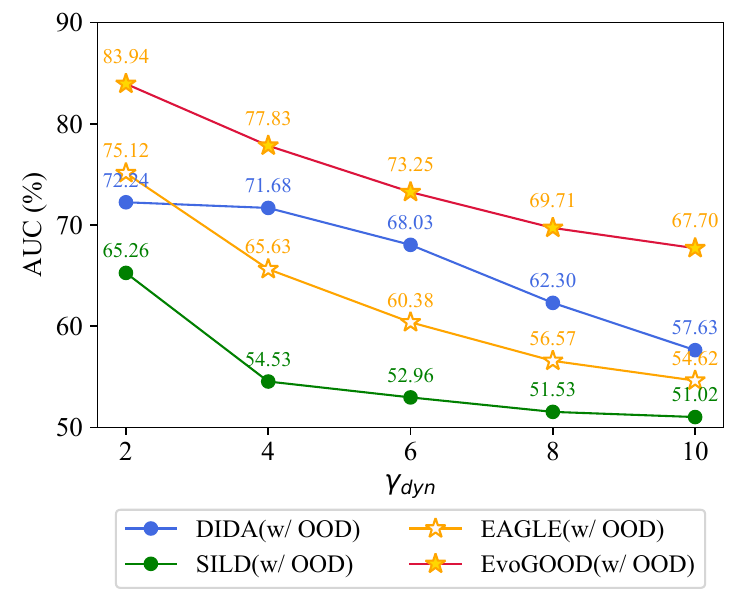}
\label{fig:evolution}
\end{minipage}
}

\centering
\caption{Investigation on Environment-aware Invariant Pattern Recognition. }
\vspace{-0.6em}
\label{fig:exp}
\end{figure*}

\subsubsection{Loss Coefficient}
We analyze the coefficient sensitivity of the ESVAE loss $\mathcal{L}_{\rm ESVAE}$ and the causal intervention loss $\mathcal{L}_{risk}$, $\beta_1$ and $\beta_2$, in Eq.~\eqref{eq:final_loss}, respectively. 
Fig.~\ref{fig:vae_loss} and Fig.~\ref{fig:risk_loss} show the results on Yelp. 
The solid line is the average AUC score in the testing stage and the shaded region is the standard deviation for 5 runs. 
The dashed line is the average AUC score of the best baseline EAGLE.

From the results, we can observe that the task performance experiences a significant decrease when $\beta_1$ and $\beta_2$ are too large or too small. 
Besides, with most parameter settings, \modelname~shows better performance than EAGLE. 
In summary, $\beta_1$ serves to mediate the trade-off between capturing the evolution of environments and estimating the distribution of environments, effectively acting as a bi-level optimization parameter. 
Similarly, $\beta_2$ plays a critical role in striking a balance between utilizing spatio-temporal invariant patterns for consistent predictions and adapting to a variety of latent environments for generalization.

\subsubsection{Efficiency of \modelname}
To compare the efficiency of \modelname~and other dynamic graph OOD methods, we show their average running time (in seconds) per epoch in Fig.~\ref{fig:efficiency}. 
We can observe that \modelname~and SILD show the best efficiency.  
\modelname~has much better efficiency than DIDA and EAGLE, which benefits from the simple dynamic graph encoding mechanism instead of the disentangled graph encoding method as in DIDA and EAGLE. 

\subsection{Investigation on Environment-aware Invariant Pattern Recognition (Q4)}
\label{subsec:invariant_recognition}

To examine the proposed mechanism for recognizing environment-aware invariant patterns, we create synthetic datasets derived from COLLAB by altering the environmental factors.
\begin{figure*}[!htbp]
\centering
\subfigure[Results on COLLAB.]{
\begin{minipage}[t]{ 0.31\linewidth}
\centering
\includegraphics[width=\linewidth]{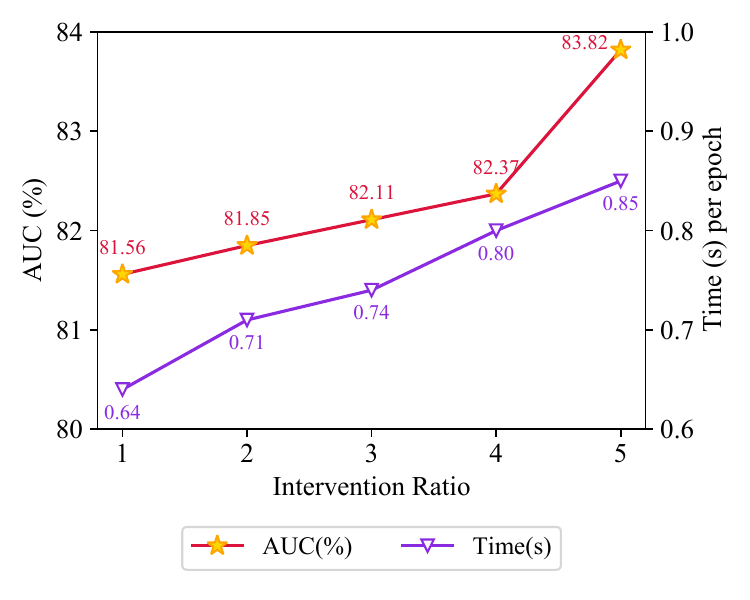}
\end{minipage}
}
\subfigure[Results on Yelp.]{
\begin{minipage}[t]{ 0.31\linewidth}
\centering
\includegraphics[width=\linewidth]{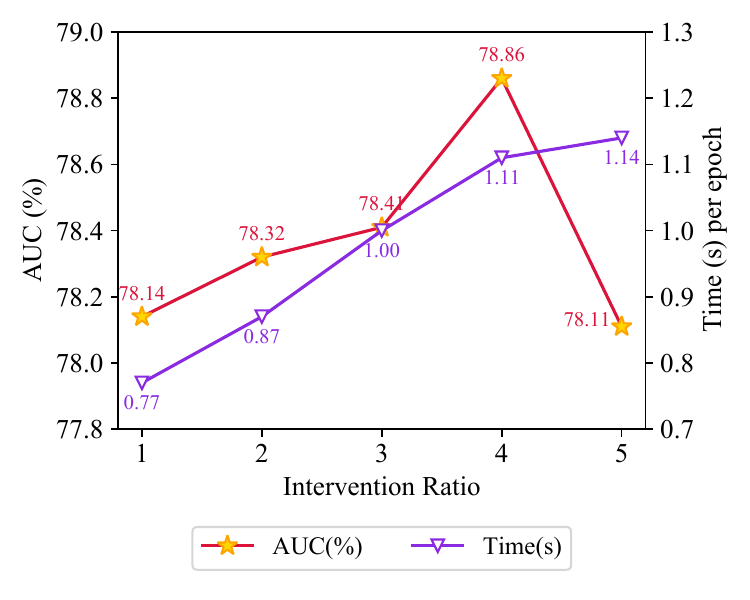}
\end{minipage}
}
\subfigure[Results on ACT.]{
\begin{minipage}[t]{ 0.31\linewidth}
\centering
\includegraphics[width=\linewidth]{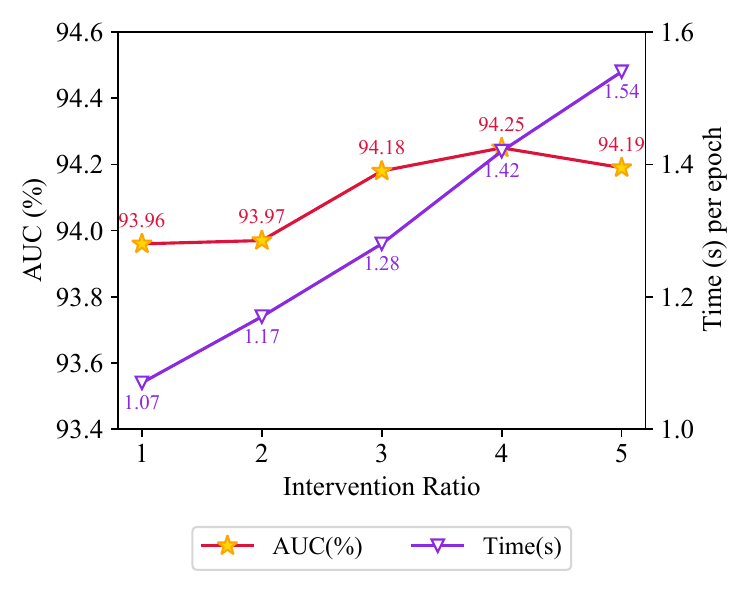}
\end{minipage}
}
\centering
\vspace{-0.6em}
\caption{AUC and efficiency with different intervention ratios. }
\label{fig:ratio}
\end{figure*}

\subsubsection{Stationary Environments. }
We assume the environments are affected by $K$ hidden variables and let $\gamma_{inv}$ denote the ratio of the environments in which the invariant patterns are learned. 
As $\gamma_{inv}$ increases, the reliability of the invariant patterns strengthens. 
Node features for different variables are drawn from $K$ multivariate normal distributions $\mathcal{N}(\boldsymbol{\mu}_{K_i};\boldsymbol{\sigma}_{K_i})$. 
Features associated with invariant patterns pertinent to $K_i$ will undergo minimal perturbation, while the opposite is true for other features.
The graph structures are constructed based on node feature similarity. 
For the OOD settings, we filter out links built under a certain $K_i$ during the training and validation stages as in Section~\ref{subsec:link_attribute}. 
We compare \modelname~with the strongest baselines DIDA~\cite{zhang2022dynamic}, SILD~\cite{zhang2023spectral} and EAGLE~\cite{yuan2023environment}.

The results are shown in Fig.~\ref{fig:invariant_level}, where the reported value is the mean AUC score over 5 runs. 
We can observe that, with $\gamma_{inv}$ increasing, the performance of \modelname~on OOD dataset shows a significant increase from 69.12\% to 83.04\%. 
Although DIDA exhibits a rising trend, its improving rate is more moderate. 
This suggests that DIDA is less effective at detecting changes in the underlying environments induced by varying $\gamma_{inv}$, resulting in suboptimal generalization performance. 
In contrast, \modelname~is able to leverage more dependable invariant patterns and execute high-quality causal interventions, thereby achieving superior generalization capabilities.

\subsubsection{Non-stationary Environments. }
We assume the environments are affected by a static factor and a dynamic factor in Section~\ref{subsec:ESVAE}. 
The node features with respect to the static factor are drawn from a multivariate normal distribution $\mathcal{N}(\boldsymbol{\mu}_{sta};\boldsymbol{\sigma}_{sta})$. 
Node features with respect to the dynamic factor are drawn from $\sin (4T) *\mathcal{N}(\boldsymbol{\mu}_{dyn};\boldsymbol{\sigma}_{dyn})$, where $T$ is the period of the corresponding dataset. 
Let $\gamma_{dyn}$ represent the dynamic level of the environments. 
Then these two types of features are mixed with the dynamic weight $\gamma_{dyn}$. 
The higher $\gamma_{dyn}$ is, the more dynamic the environments are.

We compare \modelname~with the strongest baselines DIDA~\cite{zhang2022dynamic}, SILD~\cite{zhang2023spectral} and EAGLE~\cite{yuan2023environment}. 
Fig.~\ref{fig:evolution} shows the mean AUC scores over 5 runs. 
We can observe that with $\sigma_{dyn}$ increasing, all methods' performance decreases, indicating that the dynamic graphs with a larger evolving level are more challenging to generalize. 
Our \modelname~shows consistent advantage over the other three methods.

\subsubsection{Casual Intervention Ratio}
\label{subsubsec:intervention_ratio}
We investigate the effect of the intervention ratio when performing the fine-grained causal interventions, \ie, the intervention times to the total number of nodes $|\mathcal{V}|$. 
The changes in the AUC score and the average training time per epoch with the intervention ratio increasing on COLLAB, Yelp, and ACT are shown in Fig.~\ref{fig:ratio}.

We can observe that \modelname's performance also improved with the intervention ratio increasing to 4. 
The training time cost also increases with the intervention ratio but is still acceptable.

\section{Conclusion}
\label{sec:conclusion}
This paper introduces a novel OOD generalization framework, \textbf{\modelname}, designed for dynamic graph data.
\modelname~exploits the spatial-temporal invariant patterns by modeling the evolution of the complex non-stationary environments for the first time. 
\modelname~first learns the environment evolution by an environment sequential auto-encoder, which then employs fine-grained, node-wise causal interventions using the inferred environment distribution to identify spatio-temporal invariant patterns. 
The empirical results, derived from both real-world and synthetic datasets, reveal that \modelname~exhibits superior generalization capabilities compared to current methods.

\section*{Acknowledgments}
The corresponding author is Jianxin Li. 
This work is supported by the NSFC through grant No.62225202 and No.62302023.
This work is supported in part by NSF under grant POSE-2346158.

\ifCLASSOPTIONcaptionsoff
  \newpage
\fi

\bibliographystyle{IEEEtran}
\bibliography{ref}

\begin{IEEEbiography}
[{\includegraphics[width=1in,height=1.2in,clip,keepaspectratio]{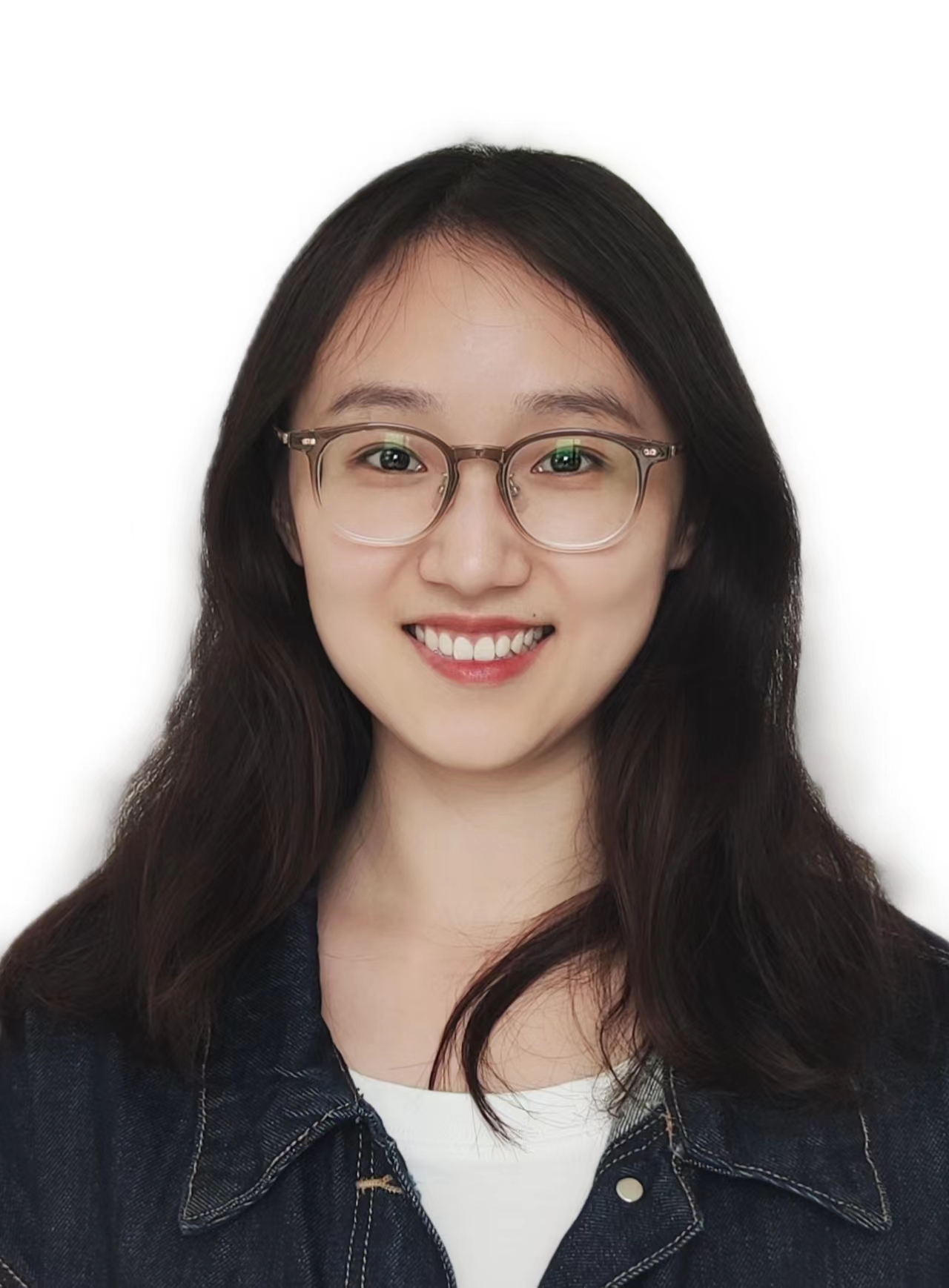}}]{Qingyun Sun} is currently an Assistant Professor at the School of Computer Science and Engineering, and Beijing Advanced Innovation Center for Big Data and Brain Computing at Beihang University. Her research interests include machine learning and graph mining. She has published several papers on IEEE TPAMI, IEEE TKDE, Web Conference, AAAI, ICDM, CIKM, etc.
\end{IEEEbiography}

\begin{IEEEbiography}[{\includegraphics[width=1in,height=1.1in,clip,keepaspectratio]{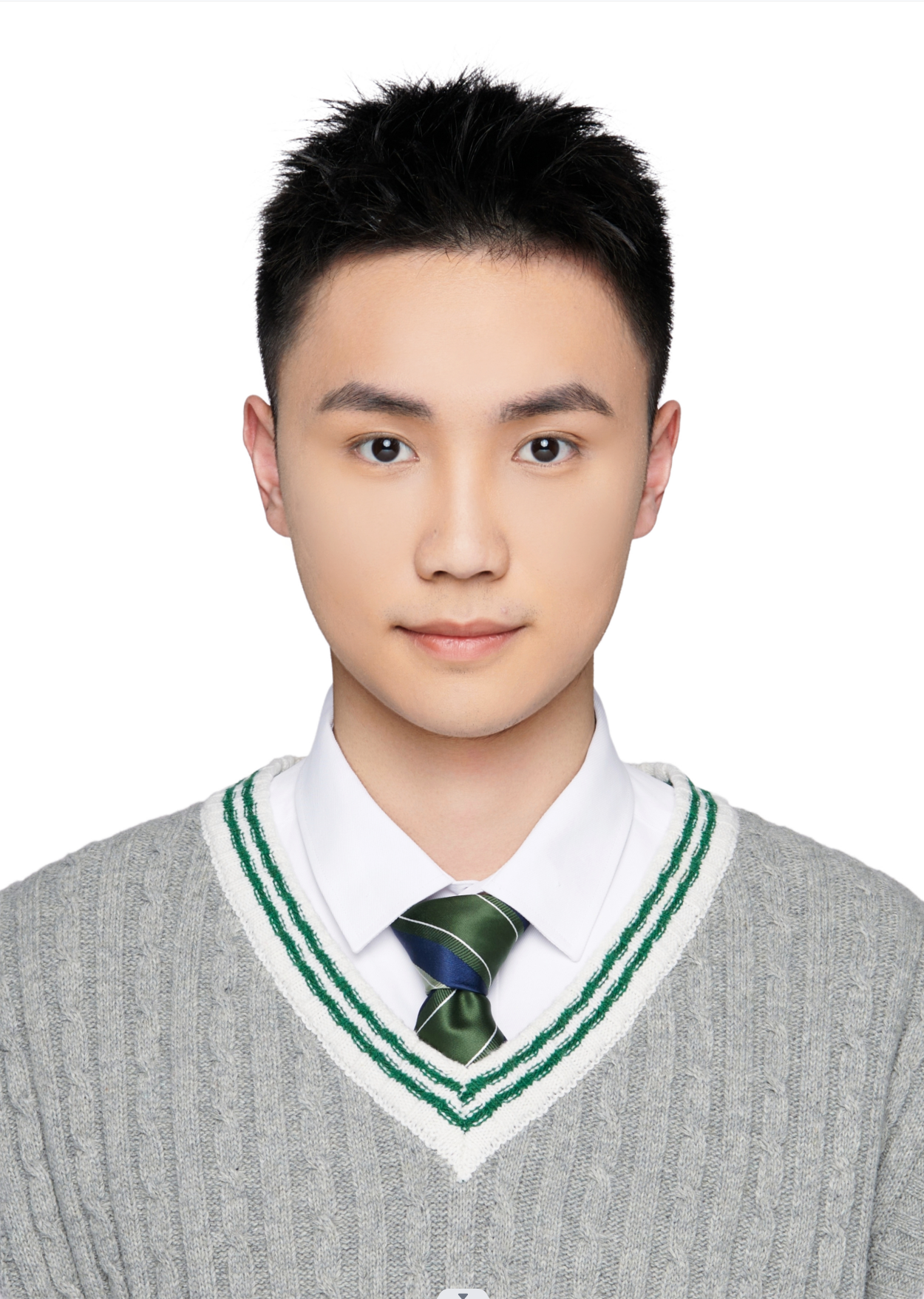}}]{Jiayi Luo}
 is currently a Ph.D. candidate at the Beijing Advanced Innovation Center for Big Data and Brain Computing at Beihang University. His research interests include graph representation learning and OOD generalization.
\end{IEEEbiography}

\begin{IEEEbiography}[{\includegraphics[width=1in,height=1.1in,clip,keepaspectratio]{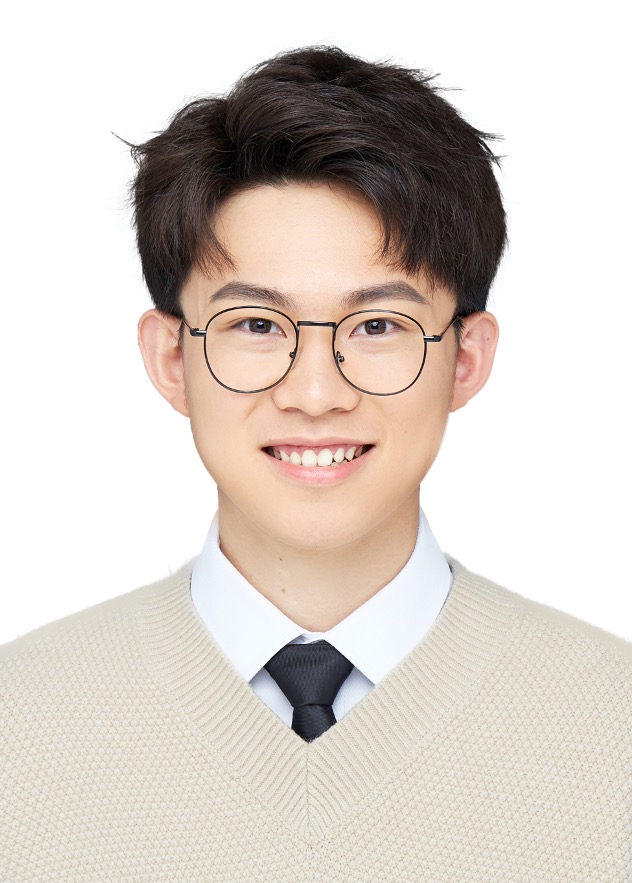}}]{Haonan Yuan}
 is currently a Ph.D. candidate at the Beijing Advanced Innovation Center for Big Data and Brain Computing at Beihang University. His research interests include dynamic graph learning and OOD generalization.
\end{IEEEbiography}

\begin{IEEEbiography}[{\includegraphics[width=1in,height=1.1in,clip,keepaspectratio]{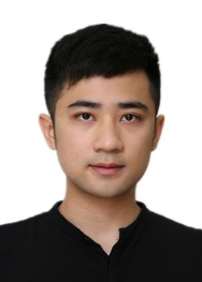}}]{Xingcheng Fu}
 is currently an assistant professor at the Key Lab of Education Blockchain and Intelligent Technology at Guangxi Normal University. His research interests include graph representation learning, complex networks, and social network analysis. 
 He has published several papers on IEEE TKDE, Web Conference, AAAI, ICDM, CIKM, etc.
\end{IEEEbiography}

\begin{IEEEbiography}
[{\includegraphics[width=1in,height=1.2in,clip,keepaspectratio]{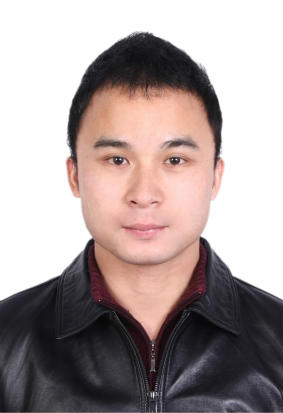}}]{Hao Peng} is currently a Professor at the School of Cyber Science and Technology in Beihang University. His research interests include representation learning, social network mining, and reinforcement learning. To date, Dr. Peng has published over 70+ research papers in top-tier journals and conferences, including the IEEE TPAMI, TKDE, TPDS, TNNLS, TASLP, JAIR, ACM TOIS, TKDD, and Web Conference. 
\end{IEEEbiography}

\begin{IEEEbiography}
[{\includegraphics[width=1in,height=1.2in,clip,keepaspectratio]{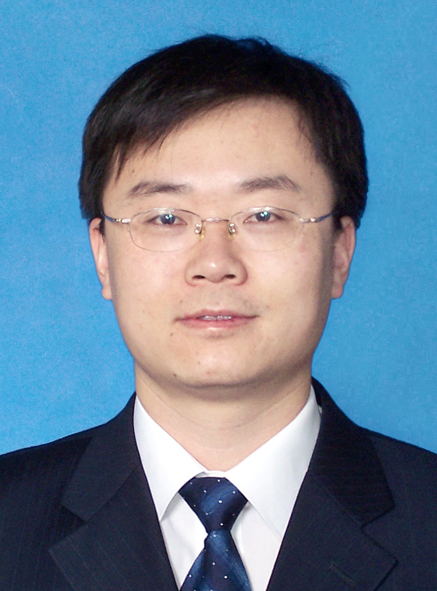}}]{Jianxin Li} is currently a Professor with the School of Computer Science and Engineering, and Beijing Advanced Innovation Center for Big Data and Brain Computing in Beihang University. His current research interests include social networks, machine learning, big data, and trustworthy computing. Dr. Li has published research papers in top-tier journals and conferences, including the IEEE TKDE, TDSC, JAIR, ACM TOIS, TKDD, KDD, AAAI, and WWW. 
\end{IEEEbiography}

\begin{IEEEbiography}
[{\includegraphics[width=1in,height=1.2in, clip,keepaspectratio]{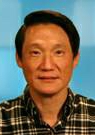}}]{Philip S. Yu} is a Distinguished Professor and the Wexler Chair in Information Technology at the Department of Computer Science, University of Illinois at Chicago. Before joining UIC, he was at the IBM Watson Research Center, where he built a world-renowned data mining and database department. He is a Fellow of the ACM and IEEE. Dr. Yu was the Editor-in-chief of ACM TKDD (2011-2017) and IEEE TKDE (2001-2004).
\end{IEEEbiography}

\end{document}